\documentclass[journal]{IEEEtran} % 或 conference

\usepackage{amsmath,amssymb,bm}

\usepackage{graphicx}
\usepackage{booktabs}
\usepackage{multirow}
\usepackage{array}
\usepackage{comment} 
\usepackage[caption=false,font=footnotesize]{subfig}
\usepackage[ruled,vlined,linesnumbered]{algorithm2e}
\usepackage[hidelinks]{hyperref}
\usepackage{cleveref}
\usepackage{xcolor}
\usepackage{soul}
\newtheorem{proposition}{Proposition}

\SetAlgoLined
\SetKwInOut{Input}{Input}
\SetKwInOut{Output}{Output}
\SetKwComment{Comment}{$\triangleright$\ }{}
\SetCommentSty{textit}  
\SetInd{1em}{1em}  
\SetNlSty{textbf}{}{}  % 行号加粗

\SetKwFor{For}{for}{do}{end}
\SetKwFor{ForEach}{for each}{do}{end}
\SetKwIF{If}{ElseIf}{Else}{if}{then}{else if}{else}{end}
\SetKw{KwTo}{to}

\begin{document}
\title{HankelKANet: Lightweight Path Loss Estimation via Physics-Inspired Hankel Convolution}
\author{Yue~Ruan, Xiaoxiao Cheng, Yunan Zhu, Wenlong Huang, Meng Hua, \textit{Senior Member, IEEE}

\thanks{Yue Ruan, Xiaoxiao Cheng, Yunan Zhu, Wenlong Huang are with the School of Computer Science and Technology, Anhui University of Technology, Maanshan 243002, China (e-mail: yue\_ruan@ahut.edu.cn, chengxiaoxiao@ahut.edu.cn, yunanzhu@ahut.edu.cn, and whuang@ahut.edu.cn.)}% <-this % 
\thanks{Meng Hua is with the Department of Electrical and Electronic Engineering, Imperial College London, SW7 2AZ London, U.K. (e-mail: m.hua@imperial.ac.uk).}
\thanks{The corresponding author is Meng Hua.}
}

% The paper headers
%\markboth{IEEE Transactions on Communication, VOL.~**, NO.~**, xxx~2026
%}{Y. RUAN \MakeLowercase{\textit{et al.}}: High-Efficiency Diode and Transistor Rectifiers}

% ====================================================================
\maketitle

% === ABSTRACT ====================================================================
% =================================================================================
\begin{abstract}
Accurate path loss (PL) estimation is essential for wireless network planning and optimization. This paper proposes HankelKANet, a lightweight physics-inspired network for PL estimation that structurally embeds electromagnetic propagation characteristics into the neural architecture. Built on a single residual UNet backbone, the proposed model integrates urban layout features with a convolutional Kolmogorov-Arnold network (KAN). Specifically, the convolutional kernels are parameterized by the radial amplitudes of the zeroth- and first-order Hankel functions. This formulation acts as a physics-inspired nonlinear distance encoding, enabling the network to capture dominant attenuation characteristics directly without requiring additional physics-based loss terms. Experimental results demonstrate that HankelKANet achieves state-of-the-art accuracy and efficiency on both the simulated RadioMapSeer dataset and the real-world Huawei RSRPSet. Compared with the leading baseline, PEFNet, it reduces the parameter count by 78\% and lowers the RMSE by 16.96\% on the simulated data, while further exhibiting robust cross-dataset transferability and high data efficiency during target-domain adaptation. To interpret these empirical gains, we mathematically establish a parameter-conditional spatial gradient bound. Using discrete Dirichlet energy (DDE) and Jacobian sensitivity heatmaps, we empirically verify that this physical encoding intrinsically constrains the network’s spatial derivatives. This structural smoothness effectively eliminates non-physical spatial ripples, preventing overfitting to simulator artifacts and serving as the fundamental driver for robust real-world generalization. \footnote{To facilitate reproducibility, the source code for this work is publicly available at: https://github.com/chengxx0317-lgtm/HankelKANet.}
\end{abstract}

% === KEYWORDS ====================================================================
% =================================================================================

% For peer review papers, you can put extra information on the cover
% page as needed:
% \ifCLASSOPTIONpeerreview
% \begin{center} \bfseries EDICS Category: 3-BBND \end{center}
% \fi
%
% For peerreview papers, this IEEEtran command inserts a page break and
% creates the second title. It will be ignored for other modes.
\IEEEpeerreviewmaketitle

% ====================================================================
% ====================================================================
% ====================================================================

% === I. INTRODUCTION =============================================================
% =================================================================================
\section{Introduction}

\IEEEPARstart{P}{ath} Loss (PL) estimation is a fundamental task in wireless communications. Accurate and efficient PL prediction plays a crucial role in the deployment and service quality assurance of 5G/6G systems, enabling critical functions such as network coverage planning, interference suppression, resource allocation, and performance optimization \cite{phillips2012survey, joseph2023path, ruan2025divided}. Consequently, developing effective and precise PL prediction schemes remains a key research priority in the field of wireless communications.

Traditional PL estimation methods, i.e., stochastic models such as 3GPP TR 36.873 \cite{3gpp36873} and TR 38.901 \cite{3gpp38901}, which rely on empirical
formulas derived from measurement data, are computationally efficient but lack adaptability to complex environments and offer limited accuracy. Deterministic models (e.g., ray tracing \cite{yun2015ray}) are capable of high-fidelity estimation by precisely simulating signal propagation paths; however, they incur extremely high computational complexity due to the modeling of numerous ray-environment interactions and cannot meet real-time requirements. Pure data‑driven models, such as ANN \cite{huang2018big} and RadioUNet \cite{levie2021radiounet}, use deep learning to uncover data correlations. However, their reliance on training data distribution and lack of physical constraints remain inherent limitations, making estimation errors likely in unseen scenarios.

To overcome the limitations of the aforementioned approaches, physics-informed neural networks (PINNs), also referred to as physics‑informed machine learning (PIML) \footnote{In this paper, PINNs and PIML are used interchangeably without distinction.}, offer an effective framework for integrating physical knowledge with data-driven learning \cite{karniadakis2021physics, meng2025physics}. By embedding physical laws, such as partial differential equations (PDEs), into the loss function, PINNs guide the model training process under physical constraints, thereby balancing the flexibility of data-driven methods with the reliability of physics-based models \cite{raissi2019physics, luo2025physics}. Building on this foundation, PEFNet \cite{jiang2024physics} fuses the volume integral equation (VIE) of electromagnetic propagation with deep learning, employing a dual-UNet architecture to separately handle physics-driven and data-driven learning. This approach achieves a promising balance between generalization and efficiency, and serves as a representative hybrid-driven method.

Despite its merits, PEFNet's complex architecture and hybrid loss function leave room for structural simplification and efficiency improvements. Recently, Kolmogorov-Arnold networks (KANs) \cite{liu2024kan} and their variants \cite{toscano2025pinns, toscano2025kkans} have offered a new way to embed physical priors by representing functions as linear combinations of learnable univariate functions. This structurally avoids the soft penalty terms required by PINNs. Notably, Hankel functions, the core components of the 2D free-space Green's function for electromagnetic propagation \cite{chew1999waves}, possess intrinsic nonlinearity. This makes them ideal learnable basis functions for a KAN-inspired, lightweight, and physics-aware design.

Motivated by this, we propose HankelKANet, an efficient architecture that directly embeds KAN-inspired learnable basis functions into convolutional kernels under the guidance of Hankel function priors.

The main contributions of this paper are summarized as follows:

\begin{itemize}
\item \textbf{A Lightweight Physics-Inspired Architecture:} We propose HankelKANet, a structurally streamlined model that processes a minimalist two-channel input (base station distance and building maps) via a single residual U-Net backbone equipped with a novel dual-convolution fusion module. This module seamlessly integrates standard spatial convolutions with Hankel-parameterized convolutions, intrinsically enforcing physical consistency without relying on complex hybrid loss functions.

\item \textbf{High Accuracy and Exceptional Efficiency:} Evaluations on both the simulated RadioMapSeer and the real-world Huawei RSRPSet demonstrate that HankelKANet achieves state-of-the-art predictive performance. Compared to the leading baseline, PEFNet \cite{jiang2024physics}, our model reduces the parameter count by 78\%, cuts training time by 26\%, and improves inference speed by 22.2\%, while simultaneously lowering the RMSE.

\item \textbf{Analytical Insights and Real-World Generalization:} We derive a parameter-conditional spatial gradient bound to provide analytical insights into the model’s inherent stability. Supported by empirical validation via discrete Dirichlet energy (DDE) and Jacobian sensitivity heatmaps, we demonstrate that this structural bias inherently enforces spatial smoothness, thereby preventing overfitting to simulator artifacts and ensuring robust cross-domain generalizability in real-world scenarios.
\end{itemize}

The remainder of this paper is organized as follows. Section~\uppercase\expandafter{\romannumeral2} reviews the methodological background. Section~\uppercase\expandafter{\romannumeral3} presents the system model. Section~\uppercase\expandafter{\romannumeral4} details the proposed HankelKANet architecture. Section~\uppercase\expandafter{\romannumeral5} reports the experimental results, including ablation and complexity studies. Section~\uppercase\expandafter{\romannumeral6} analyzes the model’s spatial smoothness and real-world generalization. Section~\uppercase\expandafter{\romannumeral7} concludes the paper.

\section{Methodology Background}
\label{sec:related}
This section establishes the methodological background. We first review PEFNet to identify its architectural and optimization bottlenecks, motivating a structurally efficient alternative. We then introduce three core components: the Hankel-function representation of wave fields (physical prior), convolutional KANs (architectural framework), and geometric complexity. Together, they provide the theoretical and analytical foundation to ensure the model’s spatial smoothness, physical plausibility, and generalization.

\subsection{PEFNet Analysis}
\label{PEFNet-analysis}

PEFNet \cite{jiang2024physics} is a pioneering hybrid model for path loss estimation that integrates fundamental electromagnetic principles with deep learning. Its core innovation is incorporating the VIE via a two-stage, dual-UNet architecture. First, a knowledge-driven UNet processes a 4-channel input to estimate the total field under VIE constraints. Then, a data-driven UNet learns the residual against measured data. By infusing physical priors, PEFNet achieves a better balance of accuracy and efficiency than pure data-driven or deterministic models.

However, PEFNet's design presents several challenges that limit its efficiency and broader applicability, directly motivating our work:
\begin{itemize}
    \item \textbf{Redundant Inputs}: Using the real and imaginary parts of the incident field as separate channels requires complex numerical preprocessing. As these fields are highly correlated with the base station distance, this 4-channel redundancy increases computational overhead and the risk of overfitting.
    
    \item \textbf{Cumbersome Architecture}: The dual-UNet structure, with a maximum channel count of 1024 and approximately 118.38M parameters, incurs massive memory and computational costs, hindering deployment on resource-constrained edge devices or large-scale real-time processing.
    
    \item \textbf{Sensitive Hybrid Loss}: The objective function ($\mathcal{L} = \mathcal{L}_{\text{PHY}} + \lambda \mathcal{L}_{\text{DAT}}$) requires meticulous manual tuning of $\lambda$. Improper weighting often triggers gradient conflicts between physical constraints and data fitting, destabilizing convergence and degrading performance \cite{wang2021understanding, rohrhofer2023data, wang2022and, bischof2025multi}.
\end{itemize}

To overcome these critical bottlenecks, Section IV introduces HankelKANet, a structurally regularized approach that inherently resolves the trade-off between physical fidelity, architectural efficiency, and optimization stability.

\subsection{Wave Field Representation via Hankel Functions}
\label{sec:hankel_physics}

Under the 2D scalar wave approximation, the electromagnetic field $u(\boldsymbol{x})$ at an observation point $\boldsymbol{x}$, excited by a transmitter at $\boldsymbol{x}_{\mathrm{Tx}}$, is strictly governed by the following integral representation:
\begin{equation}
    u(\boldsymbol{x}) = 
    \underbrace{q_{\mathrm{Tx}} G_0 \left(\boldsymbol{x}, \boldsymbol{x}_{\mathrm{Tx}}\right)}_{\text{\scriptsize Primary Direct Field}} 
    + 
    \underbrace{\int_{\mathcal{S}} G_0 \left(\boldsymbol{x}, \boldsymbol{y}\right) J(\boldsymbol{y}) \,\mathrm{d}\boldsymbol{y}}_{\text{\scriptsize Secondary Scattered Field}},
    \label{eq:resp_field_decomp}
\end{equation}
where $G_0(\boldsymbol{x},\boldsymbol{y}) = \frac{\mathrm{i}}{4} H_0^{(1)}\left(k_0 \|\boldsymbol{x}-\boldsymbol{y}\|_2\right)$ is the free-space Green's function with the background wavenumber $k_0$, and $J(\boldsymbol{y})$ represents the equivalent induced sources over the scatterer boundaries $\mathcal{S}$.

This decomposition explicitly highlights the fundamental role of Hankel functions in wave propagation \cite{zhang2002fast}. The primary direct field (the first term) is dictated by the zeroth-order Hankel function of the first kind, $H_0^{(1)}$. Physically, $H_0^{(1)}$ models an isotropic, radially symmetric cylindrical wave, whose magnitude governs the fundamental geometric energy spreading in free space. Furthermore, the local spatial gradient of this primary wave is analytically characterized by the first-order mode $H_1^{(1)}$, based on the derivative relationship $\frac{\mathrm{d}}{\mathrm{d}z} H_0^{(1)}(z) = -H_1^{(1)}(z)$. Conversely, the integral term representing the secondary scattered field accounts for complex phenomena such as scattering and diffraction from environmental boundaries. In classical analytical methods, evaluating this integral requires expanding the field into an infinite series of higher-order angular harmonics (multipole modes, $H_n^{(1)}$ for $n \ge 2$), which capture the intricate angular dependencies and directional fluctuations induced by environmental obstacles \cite{gonzalez2021advanced}.

While Eq.~\eqref{eq:resp_field_decomp} provides an exact boundary-value formulation, explicitly evaluating the full scattering integral and its corresponding high-order multipole expansions over electrically large urban layouts is computationally intractable. However, this rigorous physical decomposition successfully separates the unperturbed primary radial propagation from the complex secondary angular scattering. It naturally lays the theoretical foundation for extracting a highly structured, physics-inspired spatial prior for macroscopic channel modeling, as detailed in Section \ref{sec:channel_model}.

\subsection{Kolmogorov-Arnold Network}
\label{sec:KAN}

Inspired by the Kolmogorov-Arnold representation theorem \cite{kolmogorov1957representations}, Kolmogorov-Arnold Networks (KANs) replace the static linear weights of traditional neural networks with learnable univariate functions \cite{liu2024kan}. Instead of simple scalar multiplications, each connection in a KAN applies a non-linear function parameterized by a specific basis family. To construct these mappings, researchers have investigated a rich variety of bases, encompassing splines \cite{ta2024bsrbf}, frequency-domain representations like Fourier series \cite{xu2025enhancing} and wavelets \cite{seydi2024unveiling}, as well as radial \cite{chao2026rbf} and polynomial \cite{seydi2024exploring} functions.

To efficiently process spatial data, Convolutional KANs extend this paradigm by embedding learnable univariate functions into local convolution kernels \cite{bodner2024convolutional}. For an input feature map $\boldsymbol{I}$ and a $K\times K$ kernel, the convolution operation is defined as:
\begin{equation}\label{eq:ckan}
(\boldsymbol{I} \ast \boldsymbol{\Phi})_{i,j} = \sum_{p=1}^K \sum_{q=1}^K \varphi_{p,q}(\boldsymbol{I}_{i+p-1,j+q-1}),
\end{equation}
where $\boldsymbol{\Phi} = \{\varphi_{p,q}\}$ represents a set of learnable univariate basis functions associated with each spatial position $(p,q)$ in the kernel. 

By replacing fixed kernel weights with analytical basis evaluations, Convolutional KANs provide a flexible and expressive framework. More importantly for physics-aware tasks, this architecture allows for the direct embedding of specialized physical priors, such as electromagnetic wave propagation templates, into the fundamental convolution operation.

\subsection{Geometric Complexity and Jacobian Sensitivity}
\label{sect:geometric_complexity}

In the theoretical analysis of neural networks, spatial smoothness and generalization capabilities are closely tied to the network's gradient with respect to its inputs. For a learned mapping \(f_{\bm{\theta}}\), its microscopic sensitivity to local input perturbations is characterized by the Jacobian matrix, \(\mathbf{J}(\mathbf{x}) = \nabla_{\mathbf{x}} f_{\bm{\theta}}(\mathbf{x})\). The Frobenius norm of this matrix, \(\|\mathbf{J}(\mathbf{x})\|_F\), quantifies the local rate of change at a specific spatial coordinate \(\mathbf{x}\). A bounded Jacobian norm indicates a stable and continuous local mapping, whereas excessively large values imply abrupt, high-frequency spatial fluctuations.

To evaluate the overall smoothness of a model from a macroscopic perspective, geometric analysis utilizes the Discrete Dirichlet Energy (DDE) \cite{dherin2022neural}. The DDE quantifies the global geometric complexity by aggregating the squared local Jacobian norms over an entire dataset \(D\):
\begin{equation}
\langle f_{\bm{\theta}}, D \rangle = \frac{1}{|D|}\sum_{\mathbf{x}\in D} \|\nabla_{\mathbf{x}} f_{\bm{\theta}}(\mathbf{x})\|_F^2.
\label{eq:dirichlet_energy}
\end{equation}
Originating from harmonic function theory \cite{mitrea2020geometric, yueh2023theoretical}, the DDE measures the average functional variation across the input distribution. Minimizing this energy encourages networks to find solutions with minimal intrinsic stretching, which has been theoretically linked to improved generalization \cite{dherin2022neural}. Recently, this concept has also been adopted to analyze the functional properties and training dynamics of KAN architectures \cite{toscano2025kkans}.

In the context of wireless communications, these mathematical properties carry direct physical significance. Real-world wave propagation, including macroscopic path loss and diffraction, exhibits inherent spatial coherence. Consequently, physically plausible radio maps naturally demand spatial continuity. A neural network that generates highly irregular spatial gradients, which mathematically manifest as localized Jacobian hotspots and an anomalously high global DDE, fundamentally violates electromagnetic principles. Such behavior suggests that the network is fitting non-physical spatial noise instead of the actual wave propagation laws.

\section{System Model}
\label{sec:system}

This section formulates the map-based path loss estimation problem in urban wireless propagation scenarios. We first introduce the simplified large-scale channel model, followed by the map-based scenario representation and the learning objective.

\subsection{Large-Scale Channel and Path Loss Model}
\label{sec:channel_model}

We now formalize the connection between the fundamental wave field representation in Section~\ref{sec:hankel_physics} and the macroscopic radio channel modeling. In wireless communications, considering a transmitter (Tx) at $\mathbf{p}_{\rm Tx} \!\in\!\mathbb{R}^{3}$ and a receiver (Rx) at $\mathbf{p}\!\in\!\mathbb{R}^{3}$, the baseband multi-path channel is typically represented as a superposition of $L$ physical propagation paths:
\begin{equation}
\mathcal{H}(\mathbf{p};t,f) = \sum_{l=1}^{L} c_{l}(\mathbf{p})\,
e^{j\psi_{l}}\,e^{j2\pi\nu_{l}t}\,e^{-j2\pi f\tau_{l}},
\label{eq:multipath_TF}
\end{equation}
where $c_{l}(\mathbf{p})\!\in\!\mathbb{C}$, $\psi_{l}$, $\tau_{l}$, and $\nu_{l}$ denote the complex amplitude, carrier-induced phase, delay, and Doppler shift of the $l$-th path, respectively. Physically, the complex amplitude $c_{l}(\mathbf{p})$ of each path is the macroscopic manifestation of the primary direct field and the secondary scattered field integrations defined in Eq.~\eqref{eq:resp_field_decomp}.

Since the carrier phase $\psi_{l}$ is highly sensitive to distance variations on the scale of a wavelength, $\{\psi_{l}\}$ can be modeled as i.i.d.\ uniform random variables on $[0,2\pi)$. By averaging out these local phase fluctuations, Doppler shifts, and small-scale spatial variability, the large-scale power gain becomes $\beta(\mathbf{p})=\mathbb{E}[|\mathcal{H}(\mathbf{p})|^{2}]=\sum_{l=1}^{L}\mathbb{E}[|c_{l}(\mathbf{p})|^{2}]$, as cross-terms vanish due to phase independence. The targeted large-scale PL is thus defined as:
\begin{equation}
\mathrm{PL}(\mathbf{p}) = -10\log_{10}\beta(\mathbf{p}),
\label{eq:path_loss_def}
\end{equation}
which inherently represents a spatial local average, strictly excluding small-scale fast fading by design.

To make urban path loss modeling computationally tractable at the city scale, we introduce two strategic simplifications in our network design based on the theoretical decomposition in Eq.~\eqref{eq:resp_field_decomp}. First, we adopt a widely used 2D spatial formulation ($\boldsymbol{x}\!\in\!\mathbb{R}^{2}$). This is justified by the practical observations that the Tx is mounted at a fixed height with Rx points on a horizontal user-height plane, and that urban obstacles (e.g., buildings) are well-approximated as vertically extruded prisms. However, even under this 2D approximation, explicitly evaluating the full boundary integration for secondary scattering remains computationally prohibitive. Therefore, instead of a brute-force numerical integration, our architecture relies on a physical decoupling strategy: we mathematically evaluate the Hankel functions on the absolute transmitter-to-receiver distance map ($\|\boldsymbol{x} - \boldsymbol{x}_{\mathrm{Tx}}\|$) to isolate the envelope of the primary direct field as a global spatial prior. The complex secondary scattering and diffraction interactions are then left to be implicitly learned by a parallel network branch operating on the building layouts.

\begin{figure*}[!t]
\centering
\subfloat[]{\includegraphics[scale=0.241]{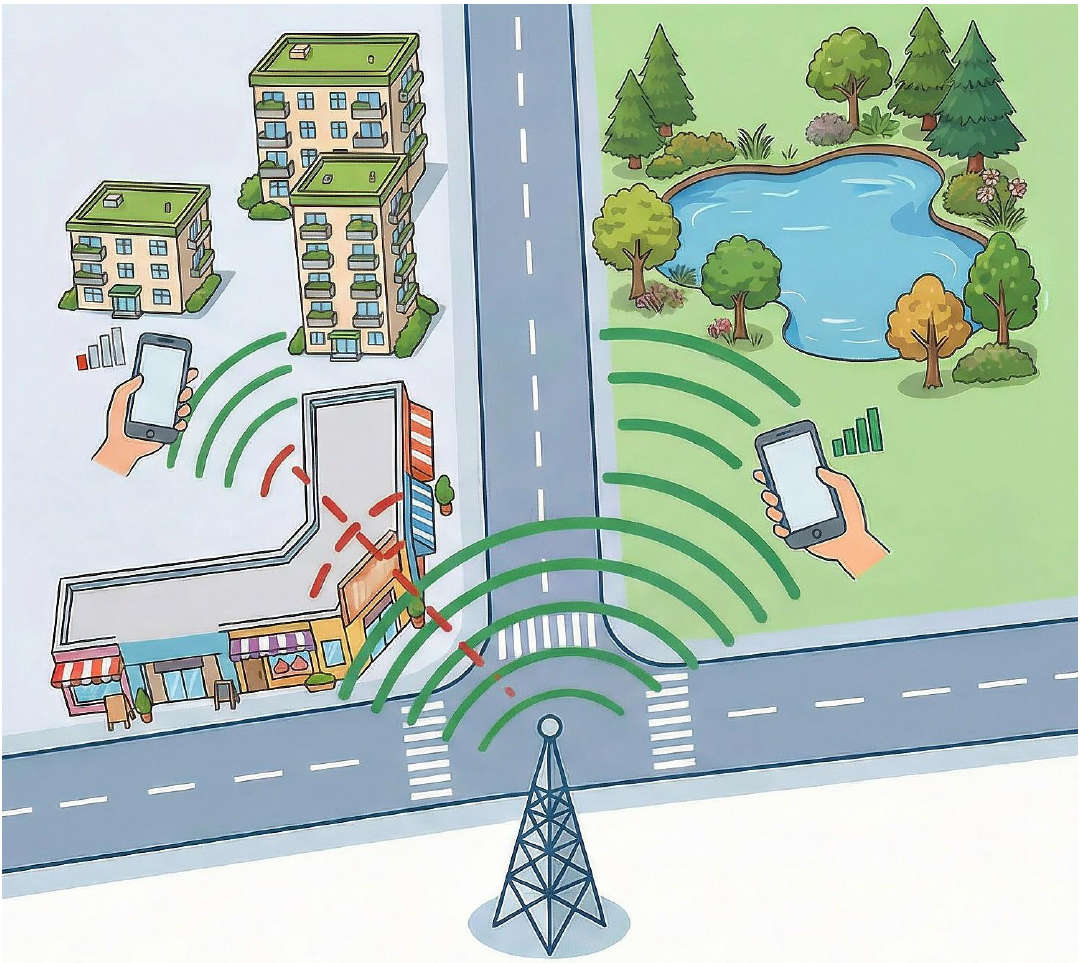}}
\hspace{5mm}
\subfloat[]{\includegraphics[scale=0.245]{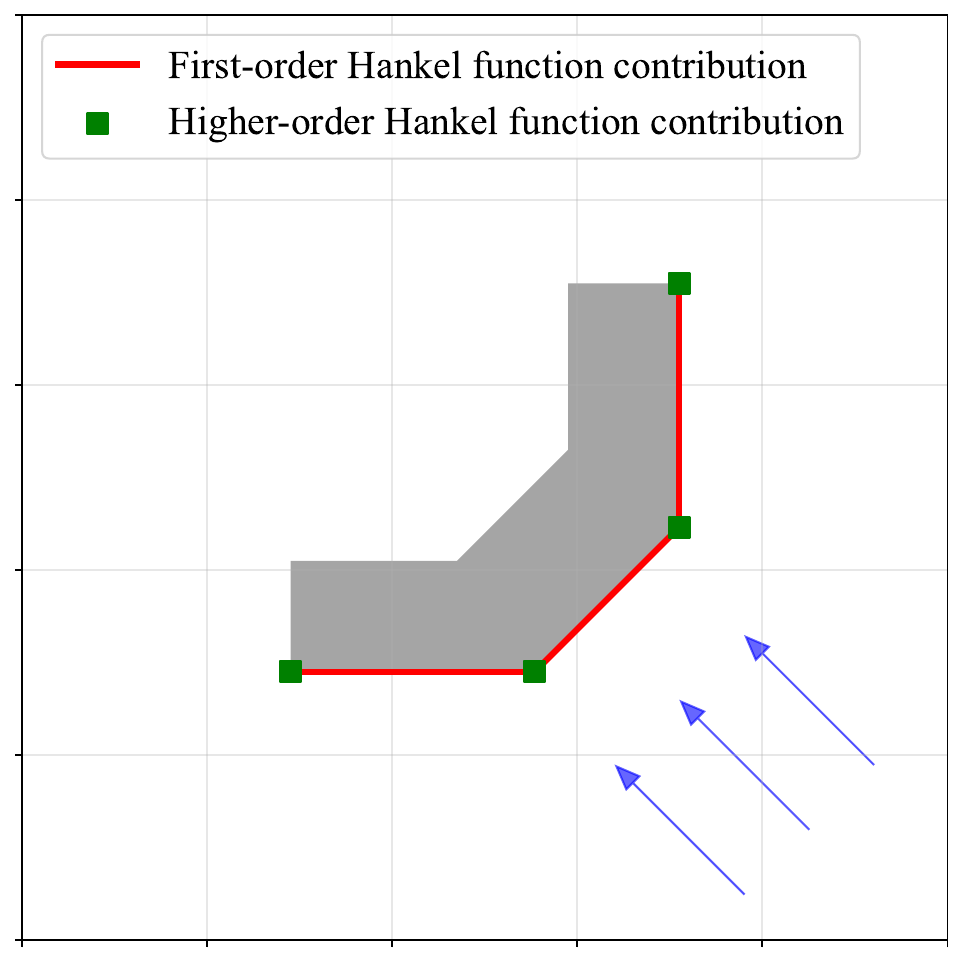}}
\hspace{5mm}
\subfloat[]{\includegraphics[scale=0.378]{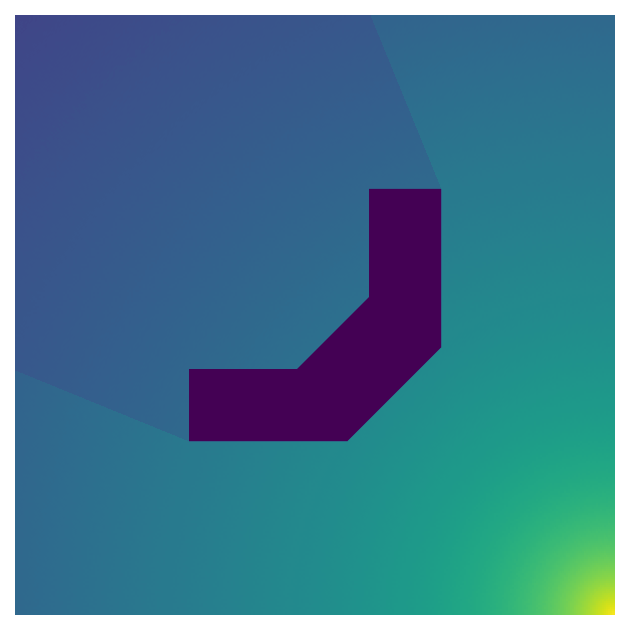}}
\caption{Electromagnetic interpretation of the structured Hankel prior for large-scale path loss. (a) A typical urban scenario exhibiting sharp PL contrasts between dense buildings and open areas. (b) The zeroth- and first-order Hankel magnitudes capture macroscopic cylindrical spreading and fundamental spatial gradients, whereas highly oscillatory higher-order modes ($n\!\geq\!2$) correspond to fine-grained scattering that averages out in large-scale PL. (c) Visualization of the combined effect of the zeroth- and first-order Hankel magnitudes extracted as the spatial prior.}
\label{fig:hankel_vis}
\end{figure*}

Building upon this strategy, we construct a low-order propagation prior specifically tailored for large-scale PL. As illustrated in Fig.~\ref{fig:hankel_vis}, we approximate the spatial prior using exclusively the magnitudes of the zeroth- and first-order Hankel functions based on the following physical considerations:
\begin{itemize}
    \item \textbf{Macroscopic Energy Decay:} The magnitude of the zeroth-order mode ($|H_0^{(1)}|$) directly models the baseline geometric spreading of the primary wave, establishing the fundamental distance-dependent decay envelope.
    \item \textbf{Dominant Spatial Gradients:} The magnitude of the first-order mode ($|H_1^{(1)}|$) captures the dominant radial gradients of the primary wave. Supplying these gradients as a spatial prior equips the neural network to effectively learn abrupt macroscopic power transitions, such as shadow boundaries and edge diffraction induced by the building layouts.
    \item \textbf{Negligibility of Higher-Order Modes:} Higher-order modes ($n \ge 2$) represent highly oscillatory multi-bounce scattering. While driving small-scale fading, these rapid fluctuations are intrinsically smoothed out by the local spatial averaging in Eq.~\eqref{eq:path_loss_def}, making their macroscopic impact negligible.
    \item \textbf{Omission of the Phase Argument:} Because PL focuses solely on the averaged power envelope rather than the coherent superposition of fields, the phase argument of the Hankel functions becomes redundant for describing large-scale fading.
\end{itemize}

By discarding the phase arguments and higher-order scattering terms, we extract a highly structured, physics-inspired, and computationally lightweight spatial prior. The rigorous empirical support for this theoretical simplification is extensively detailed in the ablation study in Section~\ref{subsec:Hankel_Orders}, which demonstrates that omitting the phase and higher-order modes does not degrade the large-scale PL prediction accuracy.

\subsection{Map-Based Scenario Representation}
\label{sec:scenario_representation}

Building on the 2D formulation established above, the region of interest is discretized into a regular two-dimensional grid of size $H\times W$. A grid cell $(x,y)$ corresponds to the physical location $\mathbf{r}_{x,y}$ on the Rx-height plane, and the ground-truth path 
loss map is defined as $\mathbf{Y}\in\mathbb{R}^{H\times W}$ with $\mathbf{Y}(x,y) = PL(\mathbf{r}_{x,y})$.

To characterize the propagation environment, we construct two spatial matrices:
\begin{itemize}
    \item \textbf{Building Layout Map ($B\in\{0,1\}^{H\times W}$):} A binary mask where $B(x,y)=1$ indicates a building (blockage) and $B(x,y)=0$ indicates free space.
    \item \textbf{Distance-Prior Map ($R\in\mathbb{R}^{H\times W}$):} Encodes the geometric distance from the transmitter, defined as $R(x,y) = \|\mathbf{r}_{x,y} - \mathbf{r}_{\rm Tx}\|_2$. In practice, this map is min-max normalized into $[0,1]$.
\end{itemize}

The input representation of a single propagation scenario is the channel-wise concatenation of these two maps:
\begin{equation}
X = \operatorname{Concat}(B,R) \in \mathbb{R}^{2\times H\times W}.
\label{eq:input_representation}
\end{equation}
Thus, the path loss estimation is cast as a two-channel-to-one-channel image translation problem, learning a mapping $\widehat{Y} = F_{\theta}(X)$ parameterized by network weights $\theta$.

\subsection{Learning Objective}
\label{sec:learning_objective}

Given a training dataset $\mathcal{D}=\{(B_i,R_i,Y_i)\}_{i=1}^{N}$, the goal is to optimize the model parameters $\theta$ to accurately predict the path loss maps. Since path loss estimation is essentially a pixel-wise continuous regression problem, the mean squared error (MSE) is adopted as the training objective:
\begin{equation}
\mathcal{L}_{\mathrm{MSE}} = \frac{1}{HW}\sum_{x=1}^{H}\sum_{y=1}^{W}\big(\widehat{Y}(x,y)-Y(x,y)\big)^2.
\end{equation}
Therefore, the final learning objective is
\begin{equation}
\label{eq:loss}
\min_{\theta}\ \mathbb{E}_{(B,R,Y)\sim \mathcal{D}} \, \mathcal{L}_{\mathrm{MSE}}.
\end{equation}

\section{Proposed Method}
\label{sec:method}
In this section, we present the architecture of the proposed HankelKANet that realizes the end-to-end mapping defined in Section~III. The model is designed to combine environmental structural information and propagation-related priors within a compact encoder--decoder framework.

\subsection{Overall Architecture}

HankelKANet adopts a UNet-based encoder--decoder backbone for end-to-end path loss prediction from the two-channel input $\{B,R\}$. This design is motivated by the nature of the task as an image-to-image regression problem, where both the input and output are two-dimensional spatial maps. The encoder progressively extracts hierarchical features and enlarges the receptive field, while the decoder restores spatial resolution and recovers fine-grained details through skip connections.

As illustrated in Fig.~\ref{fig:architecture}, the proposed network consists of two parallel components in the encoding stage: an environment branch for extracting scene-structure features from the building layout map, and a physics-prior branch for extracting propagation-aware features from the distance-prior map. The features from the two branches are fused at multiple scales and then passed to the decoder to generate the final path loss map. The entire model is built on a single residual UNet backbone with a maximum channel width of 512, resulting in a lightweight yet expressive architecture.

\begin{figure*}[!t]
  \centering
\includegraphics[width=0.95\textwidth]{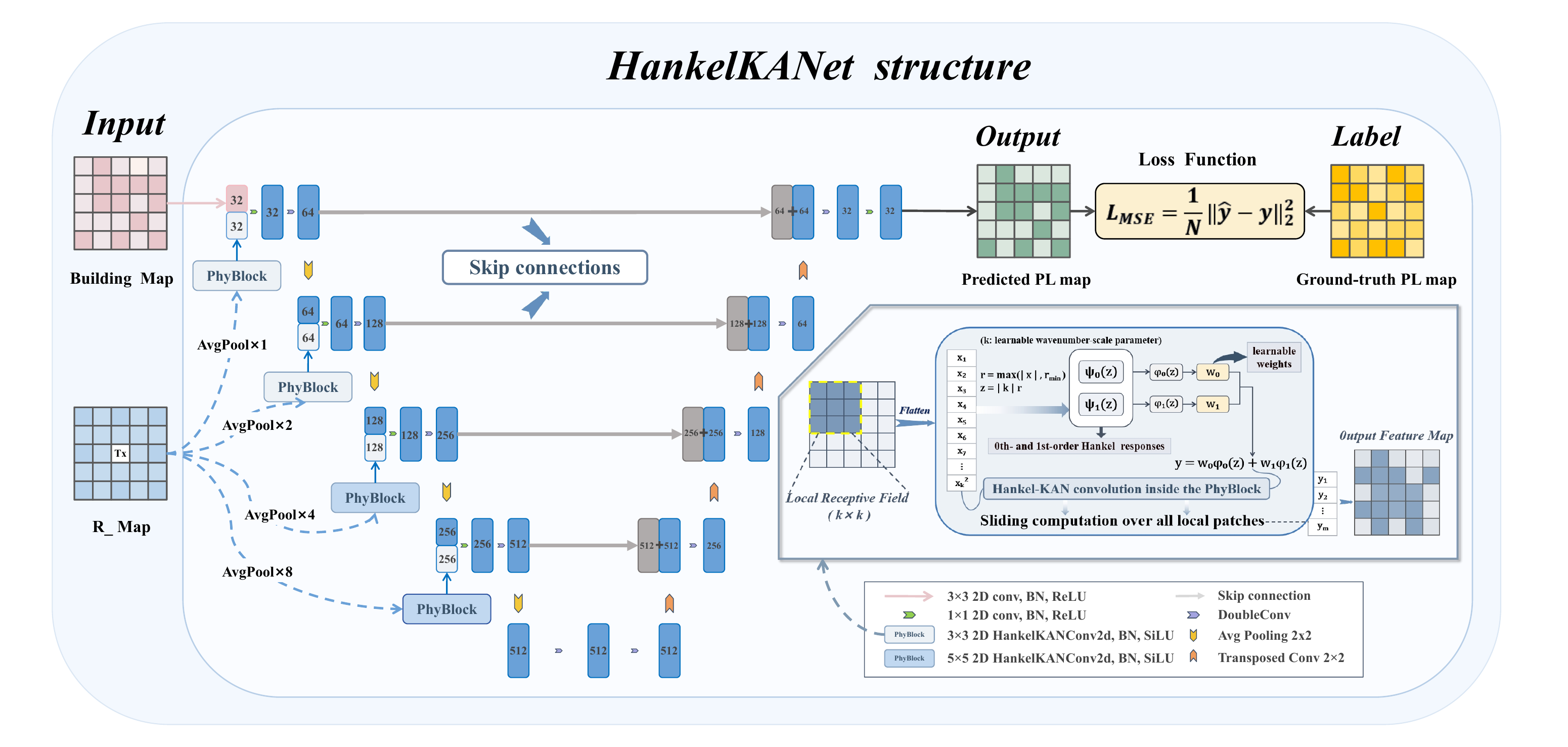}
  \caption{Overall architecture of the proposed HankelKANet. The network processes a streamlined two-channel input comprising a base station distance map (R\_Map) and a building binary map. It is built upon a single residual UNet backbone with a maximum channel count of 512. The encoder-decoder structure utilizes skip connections to recover fine details in boundary and blockage regions, while the physical branches (PhyBlocks) embed distance-dependent propagation priors via 2D HankelKAN convolutions. The network finally outputs the predicted PL map.}
  \label{fig:architecture}
\end{figure*}

\subsection{Dual-Branch Encoder Design}

\subsubsection{Environment Branch}

The environment branch takes the building layout map $B$ as input and extracts structural representations induced by the urban geometry. It is implemented by a hierarchical convolutional encoder that progressively downsamples the spatial resolution while increasing the channel dimension. In this way, the network captures both local building boundaries and broader scene-level contextual information. The intermediate features are preserved and forwarded to the decoder through skip connections for later detail recovery.

\subsubsection{Physical Branch}
\label{sec:physical_branch}

The physical branch operationalizes the propagation prior established in Section~\ref{sec:channel_model} by explicitly evaluating the primary field envelope. It takes a dimensionless distance map $\widetilde{R} = R / L_{\mathrm{ref}} \in [0,1]$ as input, where $R$ is the physical Tx-Rx distance and $L_{\mathrm{ref}}$ is the map's diagonal length. 

Crucially, to prevent the Hankel functions from diverging at their singular point ($R=0$), we lower-bound the evaluated argument as $z = \max\!\left\{ \kappa\widetilde{R}, \, z_{\min} \right\}$. Here, $\kappa$ is a dimensionless learnable radial-frequency parameter, bounded in $[0.1, 10.0]$ via a sigmoid mapping and initialized to $\pi$ to ensure numerical stability. The threshold $z_{\min} = \pi \times 10^{-3}$ restricts evaluation to a physically negligible sub-pixel region ($\widetilde{R} \ge 10^{-3}$), preventing gradient explosion caused by the divergence of the Neumann functions during backpropagation without distorting macroscopic features.

This regularized map is then processed by the proposed Hankel-KAN convolutional layer. Given an input feature map $\mathbf{X}$ and a kernel of size $m\times n$, the Conv-KAN operation is defined as:
\begin{equation}
(\mathbf{X}\!*\!\boldsymbol{\varphi})_{i,j}\;=\; \sum_{u=1}^{m}\sum_{v=1}^{n}\varphi_{u,v}\!\left(\mathbf{X}_{i+u-1,\,j+v-1}\right),
\label{eq:convkan}
\end{equation}
where each $\varphi_{u,v}(\cdot)$ is a learnable univariate function associated with the spatial position $(u,v)$. Unlike standard CNNs that apply simple scalar multiplications, Conv-KAN endows each kernel entry with an analytical, physics-inspired basis function.

Specifically, we parameterize each $\varphi_{u,v}(z)$ using a combination of the zeroth- and first-order Hankel responses:
\begin{equation}
\varphi_{u,v}(z) \;=\; w_{0,u,v}\,\phi_{0}(z) \;+\; w_{1,u,v}\,\phi_{1}(z),
\label{eq:phi_uv}
\end{equation}
where $w_{0,u,v}, w_{1,u,v} \in \mathbb{R}$ are trainable modal weights that adaptively fuse the macroscopic energy decay and the dominant spatial gradients of the primary wave. The parameter $\kappa$ (embedded within $z$) allows the network to dynamically adjust the fundamental spatial decay rate to varying propagation environments. 

Following the magnitude-only envelope extraction established in Section~\ref{sec:channel_model}, the radial basis functions $\phi_{\nu}$ ($\nu \in \{0, 1\}$) are explicitly constructed as log-compressed magnitudes. This nonlinear mapping effectively compresses the massive dynamic range of the electromagnetic field:
\begin{equation}
\phi_{\nu}(z) \;=\; \log\!\left(1 + |H_{\nu}^{(1)}(z)|\right) \;=\; \log\!\left(1 + \sqrt{J_{\nu}(z)^{2}+Y_{\nu}(z)^{2}}\right).
\label{eq:hankel_log}
\end{equation}

Consequently, the physical branch yields a structured spatial prior of the unperturbed primary field. These macroscopic features are subsequently fused with the environment branch to implicitly resolve complex secondary scattering interactions.

\subsection{Multi-Scale Feature Fusion and Decoder}

The features extracted by the environment and physical branches are fused at four progressive scales. At the first scale, the building layout map $B$ and the distance-prior map $\widetilde{R}$ are processed by an initial environment block (EnvStem) and a physical block (PhyStem), respectively. The feature extraction at this initial scale utilizes $3\times3$ convolution kernels, producing the corresponding 32-channel feature maps $E_1$ and $P_1$. These two feature maps are then concatenated and passed through a fusion layer, which can be written as
\begin{equation}
F_1=\mathcal{F}_1([E_1;P_1]), \qquad \mathcal{F}_1(\cdot)=\mathrm{Conv}^{(1)}_{1\times1}(\cdot),
\end{equation}
where $[\cdot\,;\,\cdot]$ denotes channel-wise concatenation. After this $1\times1$ fusion, a double-convolution block, consisting of a $3\times3$ convolution and a $5\times5$ convolution (each followed by batch normalization and SiLU activation), is applied to further transform the joint representation and expand the channel dimension to 64. The resulting feature $F_1$ is then downsampled and forwarded to the next scale.

For scales $s=2,3,4$, the environmental features $E_s$ are obtained by downsampling the fused features from the previous stage. Concurrently, the physical branch performs scale-matched average pooling on the original distance-prior map $\widetilde{R}$ and then applies Hankel-KAN convolutions (PhyBlock) to generate the propagation-aware features $P_s$. Specifically, the physical branch dynamically adapts its kernel sizes, using a $3\times3$ Hankel-KAN convolution at the second scale, and larger $5\times5$ Hankel-KAN convolutions at the third and fourth scales, producing 64-, 128-, and 256-channel feature maps, respectively. At each scale, the environmental and physical features are concatenated and fused through a $1\times1$ convolution, followed again by the aforementioned double-convolution block (i.e., one $3\times3$ and one $5\times5$ convolution) to refine the representation and double the channel dimension. In this way, the fused features progressively encode both scene-structure information and propagation-aware priors across multiple resolutions.

The fused representations at different scales are not only passed to deeper encoder stages, but also delivered to the decoder through skip connections. The decoder is symmetric to the encoder: it progressively upsamples the bottleneck feature and concatenates it with the corresponding encoder outputs to recover spatial resolution and reconstruct fine-grained path loss structures, especially around blockage boundaries and transition regions. As the spatial resolution increases, the channel dimension is reduced accordingly. Finally, a $1\times1$ convolution maps the decoded features to a single-channel continuous-valued output, yielding the predicted path loss map.

\subsection{End-to-End Training Strategy}

By embedding propagation-related priors directly into the network architecture through Hankel-KAN convolutions, the proposed method can be trained end-to-end using only the MSE objective (Eq.~\ref{eq:loss}). This design avoids the additional loss terms and manual coefficient tuning often required in hybrid formulations, while still enabling the model to capture both distance-dependent attenuation and environment-related spatial variations within a unified framework. The overall training procedure is summarized in Algorithm~\ref{alg:training}.

\begin{algorithm}[!t]
\label{alg:training}
\caption{Training procedure of HankelKANet}
\Input{
    Training dataset $\mathcal{D}_{\text{train}}$; validation dataset $\mathcal{D}_{\text{val}}$; Batch format: $(B,R,Y)$; model $f_0$; number of training epochs $E$; learning rate $\eta$.
}
\Output{Best model parameters $\theta^*$.}

\BlankLine
Initialize model parameters $\theta$ and the optimizer\;
\For{$\text{epoch} \leftarrow 1$ \KwTo $E$}{
    \ForEach{$(B,R,Y) \in \mathcal{D}_{\text{train}}$}{
        $E_1 \leftarrow \text{EnvStem}(B)$\;
        {\Comment{Process building map $B$ with EnvStem to get initial env features $E_1$}}
        $P_1 \leftarrow \text{PhyStem}(R)$\;
        {\Comment{Process distance map $R$ with PhyStem to get initial phy features $P_1$}}
        $F_1 \leftarrow \text{Enc}(\text{Fuse}(E_1,P_1))$\;
        {\Comment{Fuse $E_1$ and $P_1$, then encode to get first fused features $F_1$}}
        
        \For{$s \leftarrow 2$ \KwTo $S$}{
            $E_s \leftarrow \text{Pool}(F_{s-1})$\;
            {\Comment{Pool previous fused features $F_{s-1}$ for downsampling}}
            $P_s \leftarrow \text{PhyBranch}(\text{Pool}^{(s)}(R))$\;
            {\Comment{Multi-scale pool $R$, then phy branch to get $P_s$}}
            $F_s \leftarrow \text{Enc}(\text{Fuse}(E_s,P_s))$\;
            {\Comment{Fuse $E_s$ and $P_s$, then encode to get level $s$ fused features $F_s$}}
        }
        
        $F_b \leftarrow \text{Bottleneck}(\text{Pool}(F_S))$\;
        {\Comment{Pool deepest fused features $F_S$, then bottleneck to get $F_b$}}
        $D_S \leftarrow \text{Dec}(\text{Concat}(\text{Up}(F_b),F_S))$\;
        {\Comment{Upsample $F_b$, concat with $F_S$, then decode to get deepest decode features $D_S$}}
        
        \For{$s \leftarrow S$ \KwTo $2$}{
            $D_{s-1} \leftarrow \text{Dec}(\text{Concat}(\text{Up}(D_s),F_{s-1}))$\;
            {\Comment{Upsample current decode $D_s$, skip-connect with $F_{s-1}$, then decode to get $D_{s-1}$}}
        }
        
        $\hat{Y} \leftarrow \text{Conv}_{1\times1}(D_1)$\;
        Compute loss $\mathcal{L}_{\text{MSE}}$ according to Eq.~\ref{eq:loss}\;
        Update $\theta$ using the optimizer based on $\mathcal{L}_{\text{MSE}}$\;
    }
    Evaluate the model on $\mathcal{D}_{\text{val}}$\;
    \If{$\mathcal{L}_{\text{val}} < \mathcal{L}_{\text{val}}^{\text{best}}$}{
        $\mathcal{L}_{\text{val}}^{\text{best}} \leftarrow \mathcal{L}_{\text{val}}$\;
        $\theta^* \leftarrow \theta$\;
    }
}
\end{algorithm}

\section{Experimental Validation}
\label{sec:experiment}

This section evaluates the predictive accuracy, architectural contributions, and computational cost of HankelKANet. RadioMapSeer~\cite{levie2020radiomapseer} serves as the primary simulation benchmark, while Huawei RSRPSet~\cite{rsrpset16cpgmcm} provides a complementary assessment of predictive performance on real-world measurements. All experiments in this section follow a within-dataset protocol, with models trained, selected, and tested using the respective partitions of each dataset. Synthetic-to-real transfer and cross-dataset generalization are examined separately in Section~\ref{sec:analysis}. 

\subsection{Experimental Setup}
\label{subsec:experimental_setup}

\subsubsection{Dataset Description}

RadioMapSeer is a ray-tracing benchmark for urban PL prediction comprising 700 scenarios, each with 80 transmitter locations (inputs/outputs shown in Fig.~\ref{fig:example_maps}). To ensure generalization, the dataset is split at the scenario level, holding out 8,000 maps for testing. PL values are normalized to $[0,1]$ during training and converted back to dB for evaluating valid outdoor pixels. All baselines use identical partitions and evaluation masks.

Huawei RSRPSet contains real-world received power measurements, capturing practical propagation effects and measurement noise often missed by simulations. Here, we specifically compare HankelKANet against the physics-guided PEFNet. Both models are trained from scratch under matched data budgets and evaluated on a common held-out test set, computing errors strictly at valid measurement locations.

\begin{figure*}[!t]
  \centering
  \includegraphics[width=0.145\textwidth]
  {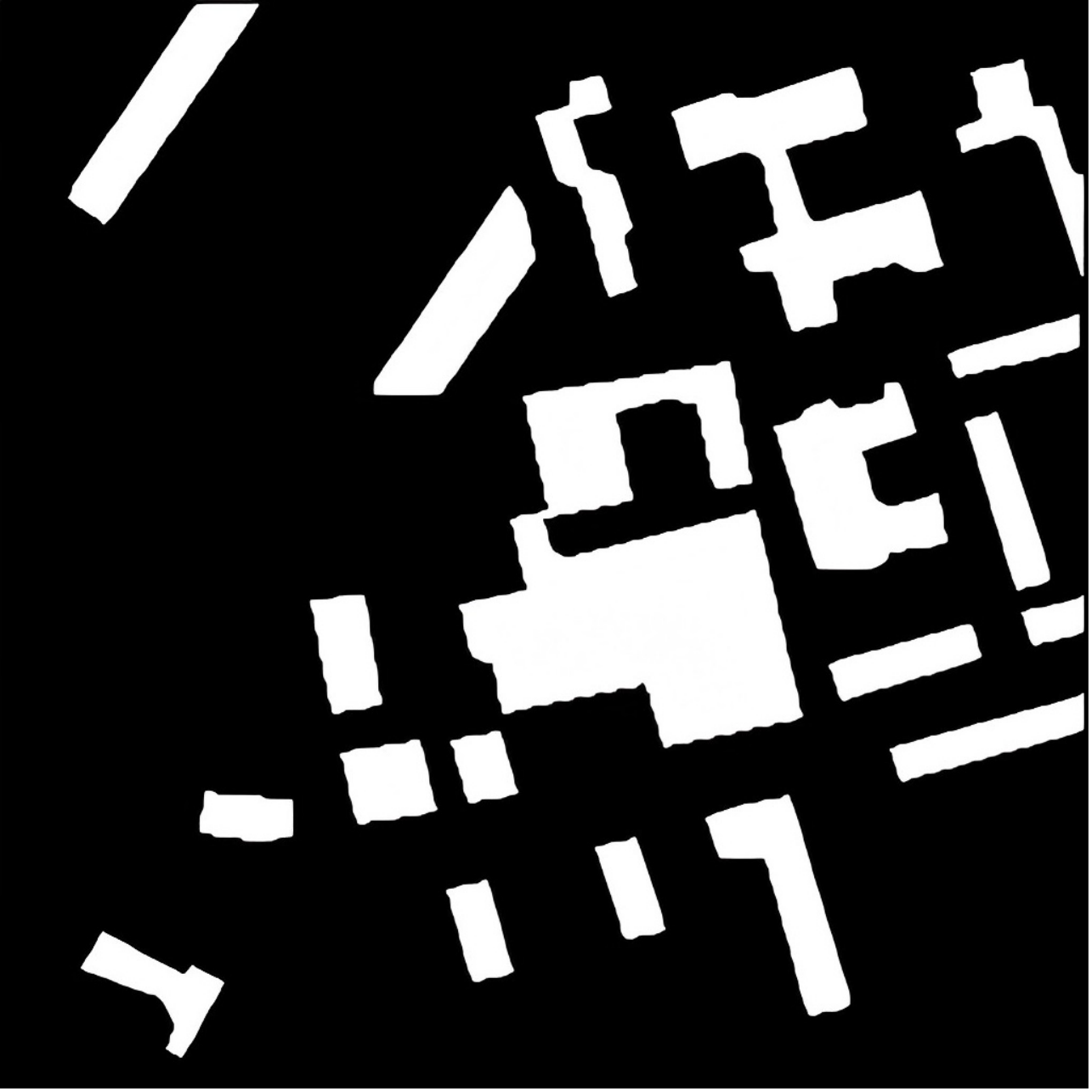}\hspace{0.01\textwidth}
  \includegraphics[width=0.145\textwidth]
  {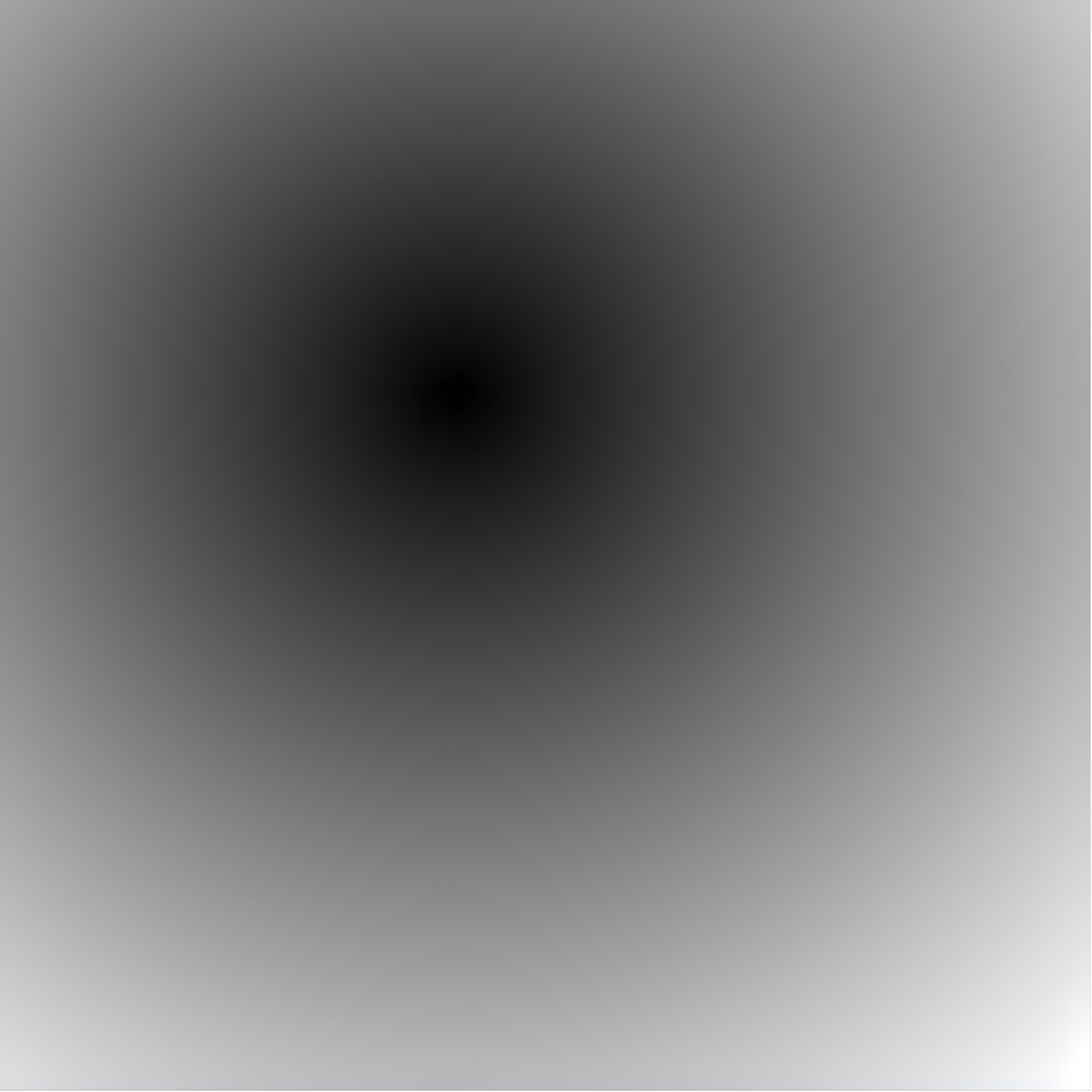}\hspace{0.01\textwidth}
  \includegraphics[width=0.145\textwidth]
  {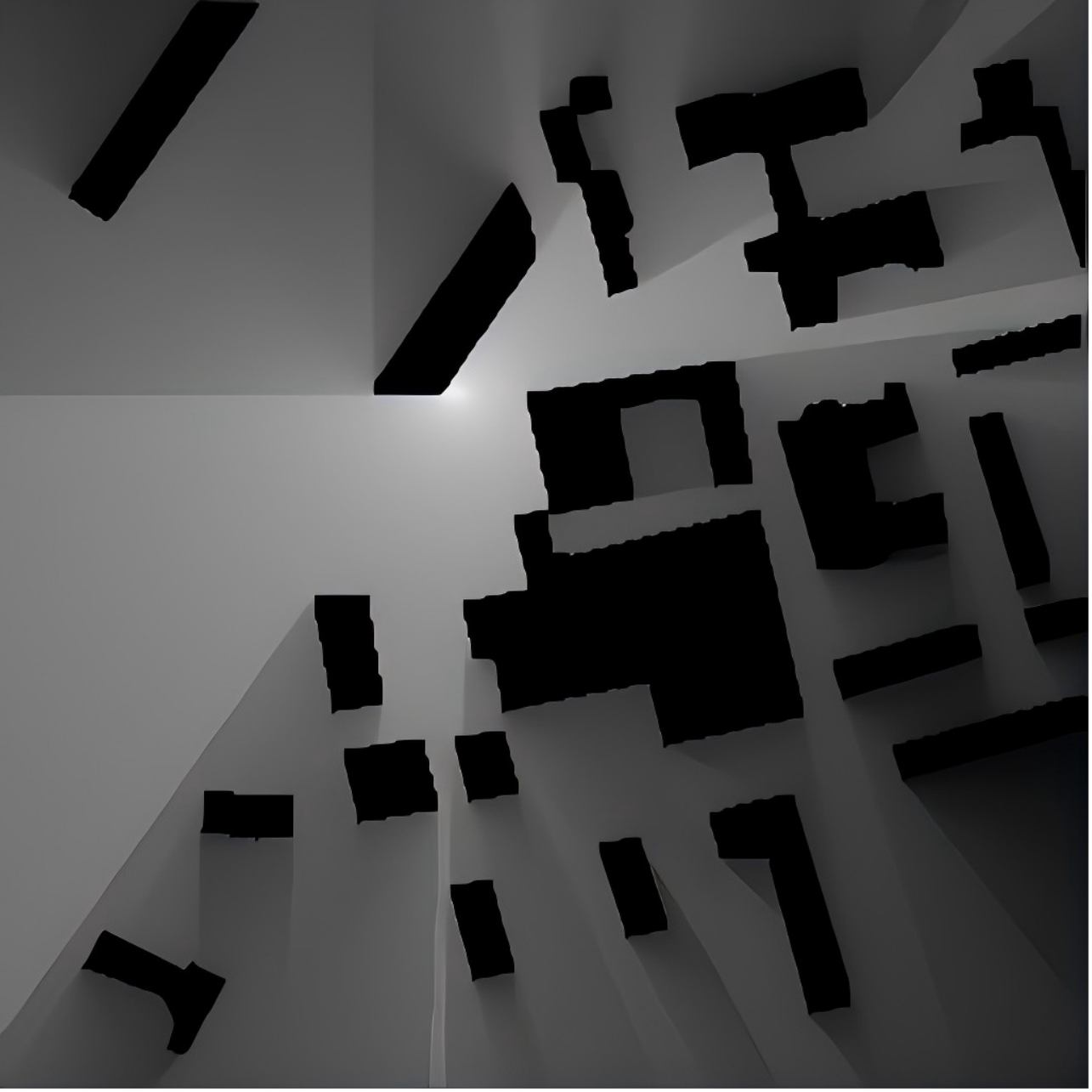}\hspace{0.02\textwidth}
  \includegraphics[width=0.145\textwidth]
  {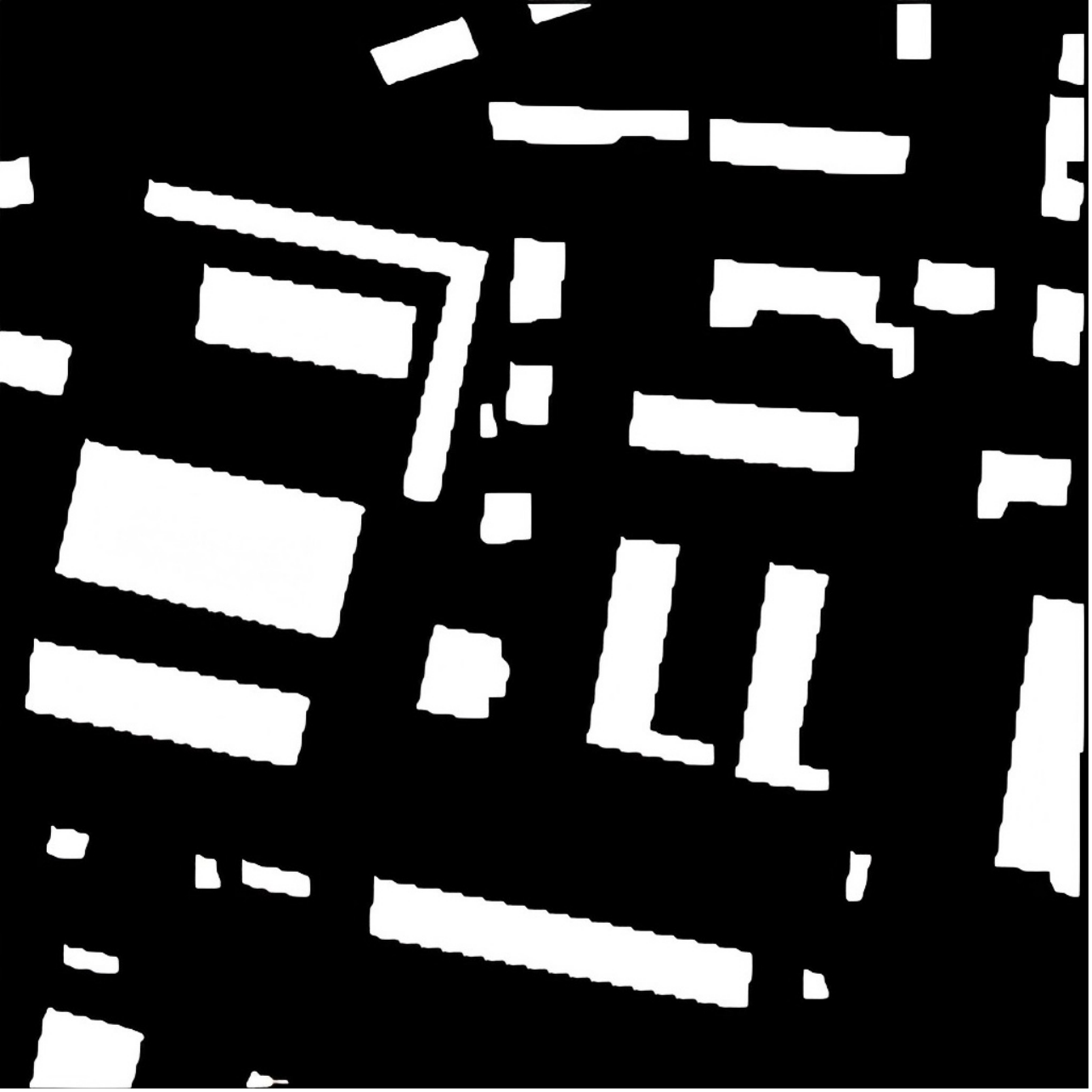}\hspace{0.01\textwidth}
  \includegraphics[width=0.145\textwidth]
  {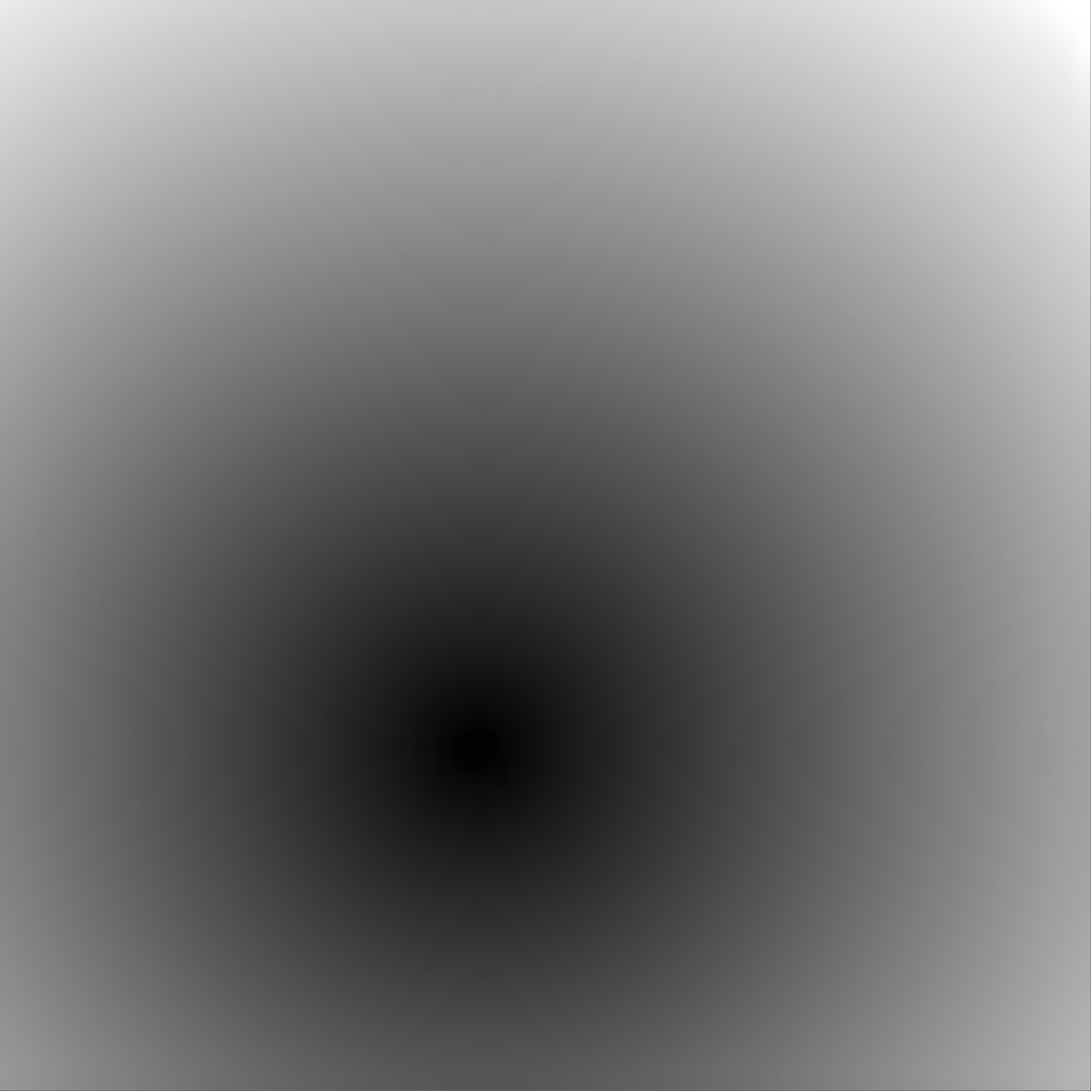}\hspace{0.01\textwidth}
  \includegraphics[width=0.145\textwidth]
  {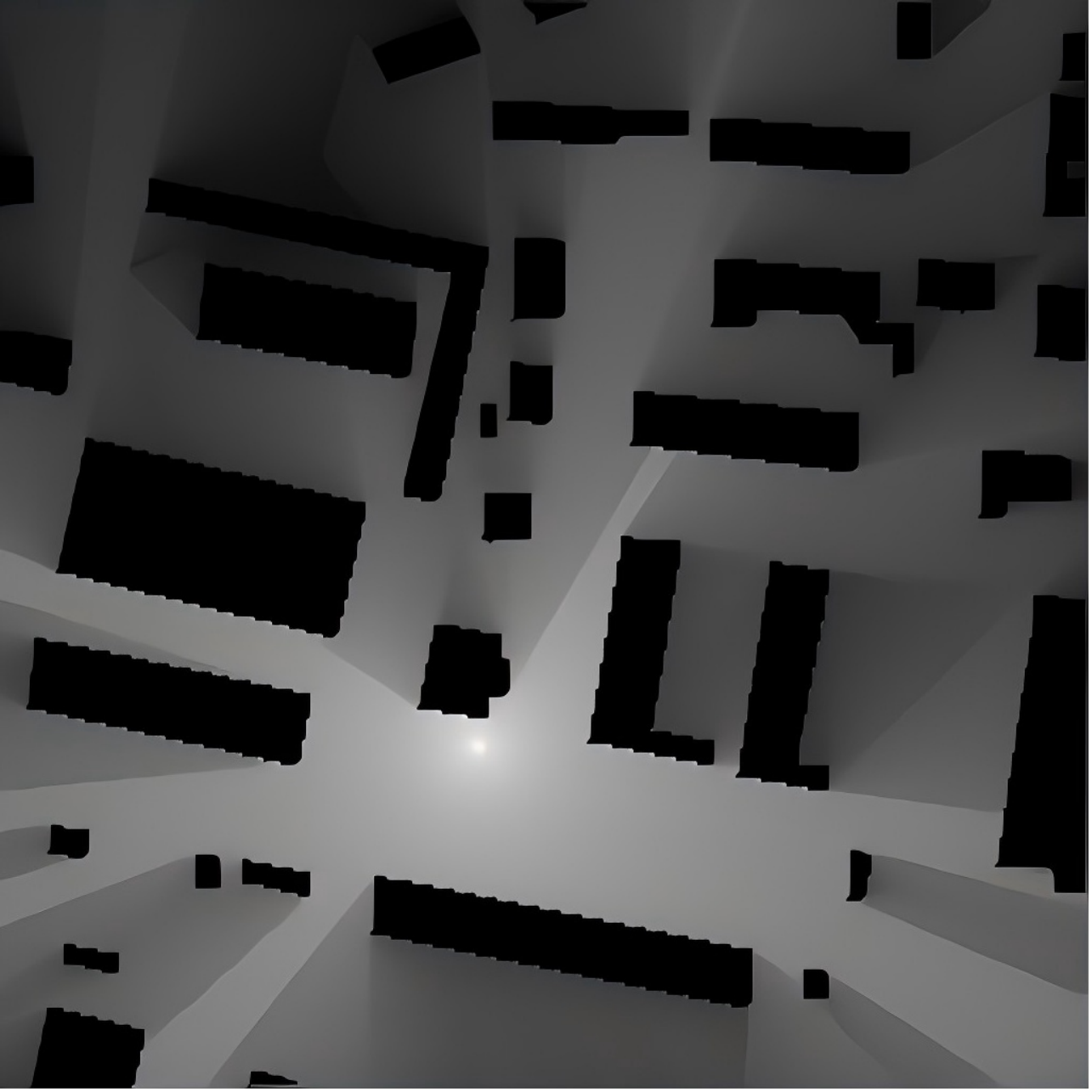}
  \caption{Examples of input and output maps from RadioMapSeer. From left to right: building layout, distance prior, and pathloss map for two representative scenarios.}
  \label{fig:example_maps}
\end{figure*}

\subsubsection{Evaluation Metrics}

We evaluate prediction performance using mean absolute error (MAE), root mean square error (RMSE), and the coefficient of determination ($R^2$). Both MAE and RMSE are reported in dB after inverse normalization. Specifically, MAE reflects the average absolute deviation, RMSE imposes a heavier penalty on large prediction errors, and a higher $R^2$ signifies better overall agreement with the reference values.

For a single test path loss map, the MAE, RMSE, and $R^2$ are computed over its valid evaluation pixels as follows:
\begin{subequations}
\label{eq:metrics}
\begin{align}
\mathrm{MAE}
&= \frac{1}{m}
\sum_{i=1}^{m}
|\mathrm{PL}_i-\hat{\mathrm{PL}}_i|,
\label{eq:MAE}\\
\mathrm{RMSE}
&= \sqrt{
\frac{1}{m}
\sum_{i=1}^{m}
(\mathrm{PL}_i-\hat{\mathrm{PL}}_i)^2},
\label{eq:RMSE}\\
R^2
&= 1-
\frac{
\sum_{i=1}^{m}
(\mathrm{PL}_i-\hat{\mathrm{PL}}_i)^2
}{
\sum_{i=1}^{m}
(\mathrm{PL}_i-\overline{\mathrm{PL}})^2
},
\label{eq:R2}
\end{align}
\end{subequations}
where $m$ is the number of valid evaluation pixels, $\mathrm{PL}_i$ and $\hat{\mathrm{PL}}_i$ denote the reference and predicted path loss values at pixel $i$, respectively, and $\overline{\mathrm{PL}}=\frac{1}{m}\sum_{i=1}^{m}\mathrm{PL}_i$ is the mean reference
path loss over these pixels. The reported test-set results are obtained by averaging each metric across all test maps, with equal weight assigned to each map.

\subsubsection{Training Settings and Reproducibility}

All experiments are implemented in Python~3.12 and run on a single NVIDIA RTX~5880 Ada GPU with 48~GB memory. HankelKANet is optimized using AdamW with an initial learning rate of $4\times10^{-3}$, a weight decay of $1\times10^{-4}$, cosine annealing, and a batch size of 32. The network is trained end-to-end using the MSE objective, without an additional physics-residual loss or manual weighting between data and physics losses. Training proceeds for up to 250 epochs, utilizing validation-based checkpoint selection and an early-stopping patience of 50 epochs.

To guarantee a strictly fair comparison, all baselines are reproduced within our unified evaluation pipeline to align perfectly with HankelKANet's inputs and training budgets. Furthermore, all models are evaluated over five independent runs (seeds 0--4) using fixed dataset partitions, with results reported as mean $\pm$ standard deviation. Computational complexity is evaluated on the aforementioned GPU using a batch size of one, a $256\times256$ resolution, and FP32 precision. For reproducibility, source code and further implementation details are available at \url{https://github.com/chengxx0317-lgtm/HankelKANet}.

\subsection{Comparison with Baselines}
\label{subsec:baseline_comparison}

We compare HankelKANet with representative empirical, deterministic, and learning-based methods, including 3GPP TR~36.873~\cite{3gpp36873}, ray tracing~\cite{yun2015ray}, decision tree regression (DTR)~\cite{zhang2019path}, $k$-nearest neighbors (KNN)~\cite{gomes2018empirical}, artificial neural networks (ANN)~\cite{huang2018big}, RadioUNet~\cite{levie2021radiounet}, and PEFNet. Among these, PEFNet serves as the primary physics-guided baseline, allowing us to compare alternative ways of incorporating propagation information into a trainable prediction model.

Note that the ray-tracing baseline is evaluated as an independently executed solver rather than the ground-truth generator. For runtime comparisons of all learning-based models, we employ a standardized inference protocol on a single GPU with a $256 \times 256$ resolution and a batch size of 1. To isolate the pure network forward-pass latency, the reported times exclude data preprocessing and are averaged over 100 repetitions following 50 warm-up iterations.

\subsubsection{Results on RadioMapSeer}
\label{subsubsec:radiomapseer_results}

Table~\ref{tab:overall_performance_tables} presents the quantitative comparison on the RadioMapSeer test set. Empirical and classical machine learning models (3GPP, DTR, KNN, ANN) exhibit limited accuracy, with RMSE values generally exceeding $10$~dB, indicating their inability to capture fine-grained spatial variations. While the independent ray-tracing solver provides a much lower RMSE of $6.3250$~dB, convolutional neural networks (RadioUNet, PEFNet) demonstrate superior performance by effectively learning spatial propagation priors from the training data.

Among all evaluated methods, HankelKANet achieves the highest $R^2$ ($0.9849$) and the lowest RMSE ($1.8392$~dB). Compared to the leading physics-guided baseline, PEFNet, our method reduces the RMSE by $16.96\%$ (from $2.2148$ to $1.8392$~dB). Interestingly, its MAE ($0.7980$~dB) is marginally higher than that of PEFNet ($0.7852$~dB). Since RMSE penalizes large errors much more heavily than MAE, this specific performance signature, a significant drop in RMSE coupled with a comparable MAE, demonstrates that HankelKANet is highly effective at suppressing severe, catastrophic prediction errors (e.g., in deep shadow regions), rather than merely pursuing a uniform reduction in absolute error across easily predictable pixels.

\begin{table*}[!t]
\centering
\caption{Quantitative performance comparison of different methods on the RadioMapSeer test set.}
\label{tab:overall_performance_tables}
\setlength{\tabcolsep}{14pt}
\renewcommand{\arraystretch}{1.65}
\begin{tabular}{lccc}
\toprule
\textbf{Method} & $\mathbf{R^2}$ & \textbf{MAE (dB)} & \textbf{RMSE (dB)} \\
\midrule
3GPP TR 36.873 \cite{3gpp36873} &0.2552 & 11.5416 & 14.7590 \\
ray tracing \cite{yun2015ray} & 0.8594 & 4.7591 & 6.3250 \\
DTR \cite{zhang2019path} & 0.4733 ($\pm$ 0.0004) & 8.7962 ($\pm$ 0.0063) & 12.1251 ($\pm$ 0.0104) \\
KNN \cite{gomes2018empirical} & 0.5808 ($\pm$ 0.0002) & 7.8994 ($\pm$ 0.0095) & 10.8649 ($\pm$ 0.0053) \\
ANN \cite{huang2018big} & 0.6303 ($\pm$ 0.0034) & 7.4167 ($\pm$ 0.0027) & 10.1688 ($\pm$ 0.0029) \\
RadioUNet \cite{levie2021radiounet} & 0.9616 ($\pm$ 0.0002) & 1.4747 ($\pm$ 0.0180) & 2.8748 ($\pm$ 0.0305) \\
PEFNet \cite{jiang2024physics} & 0.9706 ($\pm$ 0.0002) & \textbf{0.7852 ($\pm$ 0.0317)} & 2.2148 ($\pm$ 0.0492) \\
\midrule
\textbf{HankelKANet} & \textbf{0.9849 ($\pm$ 0.0002)} & 0.7980 ($\pm$ 0.0104) & \textbf{1.8392 ($\pm$ 0.0300)} \\
\bottomrule
\end{tabular}
\end{table*}

Aggregate metrics can be dominated by easily predictable pixels (e.g., short-range LOS areas), masking prediction failures in heavily obstructed or distant regions that are critical for network planning. To rigorously evaluate model robustness, Table~\ref{tab:stratified_radiomapseer} provides a stratified analysis. To ensure objectivity, the stratification thresholds are strictly data-driven, derived from training-set percentiles for link distance (short $\leq 72.80$~m, long $\geq 147.47$~m) and building density ($\geq 30.80\%$), while LOS/NLOS pixels are deterministically identified via geometric intersections.

As shown, HankelKANet maintains highly reliable estimates even in these adverse environments. While predictions in LOS and short-distance regions are remarkably accurate (RMSE $< 1.4$~dB), the model demonstrates strong robustness in NLOS regions (RMSE $2.2506$~dB) and high-density urban layouts (RMSE $2.2229$~dB). Although long-range links remain the most challenging subset ($3.1182$~dB), these stratified results confirm that the low-order Hankel representation effectively captures complex composite propagation phenomena, preserving high accuracy even under severe obstruction.

\begin{table*}[!t]
\centering
\caption{Stratified performance of HankelKANet with Hankel
orders $\{0,1\}$ against RadioMapSeer reference maps.}
\label{tab:stratified_radiomapseer}
\setlength{\tabcolsep}{7pt}
\renewcommand{\arraystretch}{1.25}
\begin{tabular}{lrrrrr}
\toprule
\textbf{Condition} &
\textbf{Radio maps} &
\textbf{Evaluated pixels} &
$\mathbf{R^2}$ &
\textbf{MAE (dB)} &
\textbf{RMSE (dB)} \\
\midrule
Overall outdoor test set &
8,000 & 397,944,539 & 0.9849 & 0.7980 & 1.8392 \\
LOS &
8,000 & 125,925,669 & 0.9885 & 0.4614 & 1.3884 \\
NLOS &
8,000 & 272,018,870 & 0.9820 & 1.2583 & 2.2506 \\
Short distance ($\leq72.80$~m) &
8,000 & 99,342,877 & 0.9910 & 0.3809 & 0.9827 \\
Long distance ($\geq147.47$~m) &
8,000 & 99,980,657 & 0.9633 & 1.9792 & 3.1182 \\
High building density ($\geq30.80\%$) &
2,160 & 89,968,345 & 0.9715 & 1.1780 & 2.2229 \\
\bottomrule
\end{tabular}
\end{table*}

\subsubsection{Results on Huawei RSRPSet}
\label{subsubsec:huawei_results}

To evaluate performance under real-world measured supervision, we deploy the models on the Huawei RSRPSet. In this section, we focus the comparison exclusively between HankelKANet and PEFNet---the strongest physics-guided baseline. Other empirical and classical machine learning methods yield substantially higher errors on this challenging real-world dataset and are therefore omitted for brevity. 

Table~\ref{tab:huawei_performance} presents the within-dataset evaluation, comparing models trained from scratch under varying fractions of target data ($\eta$). HankelKANet consistently outperforms PEFNet in suppressing large prediction errors across all data regimes. Under extremely sparse supervision ($\eta=1\%$), our model reduces the RMSE by over $0.42$~dB. With full supervision ($\eta=100\%$), HankelKANet achieves the lowest RMSE ($3.7926$~dB) and the highest $R^2$ ($0.9170$). Notably, its MAE ($1.7428$~dB) is marginally higher than that of PEFNet ($1.6865$~dB). This specific performance signature, superior RMSE coupled with comparable MAE, is perfectly consistent with the observations on RadioMapSeer (Section~\ref{subsubsec:radiomapseer_results}). It reaffirms that the compact Hankel representation excels at mitigating severe outlier errors in noisy, measurement-based real-world data, rather than overfitting to uniform absolute deviations. Further analysis regarding cross-dataset generalization is deferred to Section~\ref{sec:analysis}.

\begin{table}[!t]
\centering
\caption{Within-dataset performance comparison on the Huawei RSRPSet (Training from scratch and full real supervision).}
\label{tab:huawei_performance}
\setlength{\tabcolsep}{4.5pt}
\renewcommand{\arraystretch}{1.2}
\begin{tabular}{llccc}
\toprule
\textbf{Budget ($\eta$)} & \textbf{Method} & $\mathbf{R^2}$ & \textbf{MAE (dB)} & \textbf{RMSE (dB)} \\
\midrule
\multicolumn{5}{c}{\textit{Training from Scratch}} \\
\midrule
\multirow{2}{*}{1\%} 
& PEFNet & 0.6448 & 4.7843 & 7.8452 \\
& HankelKANet & \textbf{0.6896} & \textbf{4.4375} & \textbf{7.4218} \\
\midrule
\multirow{2}{*}{5\%} 
& PEFNet & 0.7620 & 3.5207 & 6.4211 \\
& HankelKANet & \textbf{0.7983} & \textbf{3.2884} & \textbf{5.9126} \\
\midrule
\multirow{2}{*}{10\%} 
& PEFNet & 0.8340 & 2.8129 & 5.3628 \\
& HankelKANet & \textbf{0.8593} & \textbf{2.6792} & \textbf{4.9364} \\
\midrule
\multirow{2}{*}{20\%} 
& PEFNet & 0.8792 & 2.2736 & 4.5762 \\
& HankelKANet & \textbf{0.8894} & \textbf{2.1846} & \textbf{4.3127} \\
\midrule
\multicolumn{5}{c}{\textit{Full Real Supervision}} \\
\midrule
\multirow{2}{*}{100\%} 
& PEFNet & 0.9117 & \textbf{1.6865} & 3.9124 \\
& HankelKANet & \textbf{0.9170} & 1.7428 & \textbf{3.7926} \\
\bottomrule
\end{tabular}
\end{table}

\subsection{Ablation Studies on Architecture and Physical Priors}
\label{subsec:ablation}

To systematically investigate the architectural components and the physical representation choices of HankelKANet, we conduct controlled evaluations on the RadioMapSeer benchmark.

\subsubsection{Architectural Components}

Table~\ref{tab:arch_ablation} details the impact of removing or replacing key structural components. 

Removing the physics-guided branch entirely produces a severe degradation, increasing the RMSE from $1.8392$~dB to $9.5872$~dB. This highlights the vital role of the transmitter-dependent propagation prior. Similarly, removing the U-Net skip-connected module increases the RMSE to $4.0146$~dB, confirming that the encoder-decoder feature fusion is crucial for recovering localized spatial details such as sharp shadowing transitions.

We also evaluate the specific parameterization of the physical branch. Replacing the Hankel amplitude basis with a standard B-spline basis increases the RMSE to $2.7958$~dB, demonstrating that the Bessel-derived functions provide a more effective inductive bias for electromagnetic wave propagation. Furthermore, fixing the spatial scale parameter $k$ instead of making it learnable increases the RMSE to $2.1594$~dB. A trainable $k$ allows the network to adaptively adjust the effective radial scale in the normalized feature domain, rather than being strictly bound to the physical free-space wavenumber.

\begin{table}[!t]
\centering
\caption{Architectural component ablation study on RadioMapSeer.}
\label{tab:arch_ablation}
\setlength{\tabcolsep}{6pt}
\renewcommand{\arraystretch}{1.2}
\begin{tabular}{lccc}
\toprule
\textbf{Ablation Variant} & $\mathbf{R^2}$ & \textbf{MAE (dB)} & \textbf{RMSE (dB)} \\
\midrule
w/o Physical Branch & 0.6783 & 6.7468 & 9.5872 \\
Hankel replaced by B-spline & 0.8734 & 1.7899 & 2.7958 \\
w/o skip-connected module & 0.9445 & 2.1316 & 4.0146 \\
with non-trainable $k$ & 0.8833 & 1.1926 & 2.1594 \\
\midrule
\textbf{HankelKANet (Proposed)} & \textbf{0.9849} & \textbf{0.7980} & \textbf{1.8392} \\
\bottomrule
\end{tabular}
\end{table}

\subsubsection{Hankel Orders and Angular Factors}
\label{subsec:Hankel_Orders}

To systematically determine the appropriate truncation orders and the necessity of explicit complex-field modeling, we compare varying Hankel configurations in Table~\ref{tab:physics_ablation}.

We evaluate subsets of radial orders $\mathcal{I}$. Restricting the representation to a single mode (either $\{0\}$ or $\{1\}$) degrades the RMSE noticeably ($>3.10$~dB), indicating that both the primary radial decay and its rate-of-change baseline are highly beneficial. Conversely, adding higher-order modes (e.g., $\{0,1,2\}$) yields negligible accuracy gains ($\Delta\mathrm{RMSE} = -0.0016$~dB) while increasing latency by $8.1\%$. Extending to $\{0,1,2,3\}$ actually degrades performance ($1.8662$~dB) due to parameter redundancy and optimization overhead on high-frequency oscillations. Thus, the joint low-order representation $\mathcal{I}=\{0,1\}$ serves as a highly effective configuration.

We further examine whether complete Tx-centered angular factors ($\exp(\mathrm{i}n\varphi)$) provide benefits over pure radial amplitudes. For each mode $n$, we expand the complete real and imaginary channels using Bessel functions combined with trigonometric angular terms. As shown in Table~\ref{tab:physics_ablation}, explicit angular coordinates yield marginal accuracy improvements ($\le 0.003$~dB for $|n|\le 1$) but increase runtime by up to $28\%$. Because environmental blockage primarily originates from localized building edges rather than circular wavefronts centered at the transmitter, the parallel convolutional branch already extracts these angular variations efficiently. Furthermore, since the supervision target is scalar signal power rather than complex field data, explicit phase reconstruction is largely redundant. 

Ultimately, these ablations empirically support that radial amplitudes with $\mathcal{I}=\{0,1\}$ without explicit phase modeling offers a favorable trade-off between expressive fidelity, physical adherence, and computational efficiency.

\begin{table*}[!t]
\centering
\caption{Ablation of physical representation choices (mean over 5 runs). $\mathcal{I}$ denotes the set of utilized Hankel orders.}
\label{tab:physics_ablation}
\setlength{\tabcolsep}{4pt}
\renewcommand{\arraystretch}{1.2}
\begin{tabular}{lccccc}
\toprule
\textbf{Representation Choice} & \textbf{Orders ($\mathcal{I}$) / Mode $n$} & $\mathbf{R^2}$ & \textbf{RMSE (dB)} & \textbf{FLOPs (G)} & \textbf{Latency (s)} \\
\midrule
\multicolumn{6}{c}{\textit{Radial Envelope Only (Amplitude)}} \\
\midrule
Zeroth-order only & $\mathcal{I}=\{0\}$ & 0.8559 & 3.1025 & 91.48 & 0.95 \\
First-order only & $\mathcal{I}=\{1\}$ & 0.8352 & 3.6856 & 91.48 & 0.97 \\
\textbf{Proposed Variant} & \textbf{$\mathcal{I}=\{0,1\}$} & \textbf{0.9849} & \textbf{1.8392} & \textbf{91.58} & \textbf{1.11} \\
Second-order added & $\mathcal{I}=\{0,1,2\}$ & 0.9851 & 1.8376 & 91.66 & 1.20 \\
Third-order added & $\mathcal{I}=\{0,1,2,3\}$ & 0.9845 & 1.8662 & 91.78 & 1.29 \\
\midrule
\multicolumn{6}{c}{\textit{Complete Angular Modes (Real \& Imaginary Expansion)}} \\
\midrule
Complex-field variant & $|n|\le 1$ & 0.9852 & 1.8360 & 91.66 & 1.28 \\
Complex-field variant & $|n|\le 2$ & 0.9850 & 1.8420 & 91.78 & 1.33 \\
Complex-field variant & $|n|\le 3$ & 0.9843 & 1.8750 & 91.90 & 1.42 \\
\bottomrule
\end{tabular}
\end{table*}

\subsection{Complexity and Efficiency Analysis}
\label{subsec:lightweight}

To evaluate computational efficiency, we benchmark HankelKANet against representative learning-based baselines under strictly identical inference conditions, as detailed in Table~\ref{tab:complexity}. Compared to the primary physics-guided baseline, PEFNet (evaluated via an exact local reproduction to ensure matching conditions), HankelKANet reduces the parameter count by $78\%$ (to $26.13$~M) and FLOPs by $81.9\%$ (to $91.58$~G). Peak inference memory is also reduced by over 50\% (from $720$~MB to $340$~MB), significantly lowering the barrier for deployment in resource-constrained edge environments.

Crucially, the end-to-end latency must fully account for the generation of the physical prior. Evaluating the Hankel functions online takes $0.285$~s per map. Even including this pre-processing step, HankelKANet achieves a total inference latency of $1.110$~s/map (0.901 maps/s). This remains strictly faster than both RadioUNet ($1.396$~s) and PEFNet ($1.426$~s) measured under the exact same unbatched protocol. Furthermore, HankelKANet reduces total training time on RadioMapSeer to $7.22$ hours, compared to $9.77$ hours for PEFNet. Given the nearly 3-hour per-map simulation cost of traditional ray tracing, HankelKANet offers an overwhelming efficiency advantage while simultaneously attaining superior predictive accuracy.

\begin{table*}[htbp]
    \centering
    \caption{Complexity comparison evaluated under identical conditions ($256\times256$, batch size 1).}
    \label{tab:complexity}
    \small
    \renewcommand{\arraystretch}{1.1} % 稍微缩减行高以节省空间
    \begin{tabular}{lcccccc}
        \toprule
        \textbf{Method}
        & \textbf{Params (M)}
        & \textbf{FLOPs (G)}
        & \textbf{Size (MB)}
        & \textbf{Peak Mem (MB)}
        & \textbf{Lat. (s)}
        & \textbf{maps/s} \\
        \midrule
        RadioUNet
        & 25.41 & 58.31 & 101.65 & 260 & 1.396 & 0.716  \\
        PEFNet
        & 118.38 & 506.25 & 473.52 & 720 & 1.426 & 0.701 \\
        \midrule
        Ours (Net Only)
        & 26.13 & 91.58 & 104.50 & 260 & 0.812 & 1.232 \\
        Hankel Encoding
        & 0.0005 & N/A$^{a}$ & 0.002 & 120 & 0.285 & 3.509 \\
        HankelKANet
        & \textbf{26.13} & \textbf{$\geq 91.58^{a}$} & \textbf{104.50} & \textbf{340} & \textbf{1.110} & \textbf{0.901} \\
        \bottomrule
    \end{tabular}

    \vspace{1mm}
    \begin{minipage}{0.85\linewidth}
        \footnotesize
       $^{a}$ Standard FLOP profilers only count basic neural network operations and cannot measure the mathematical complexity of special functions (e.g., Bessel functions). Therefore, 91.58 G represents the FLOPs for the network backbone only. The actual computational cost of the Hankel encoding is instead transparently reflected in its wall-clock latency (0.285 s).
    \end{minipage}
\end{table*}

\vspace{2mm}
Ultimately, these evaluations confirm that HankelKANet attains superior predictive accuracy and an overwhelming efficiency advantage. Notably, this performance is validated within a well-defined operating regime targeting macroscopic large-scale path loss for 2D outdoor urban scenarios at sub-6~GHz frequencies. Given link distances up to $\sim$350~m, this macroscopic context naturally justifies the truncation to low-order Hankel modes. Having established the empirical capabilities and operational bounds of the proposed architecture, Section~\ref{sec:analysis} proceeds to analyze its learned spatial representations and cross-dataset generalization.

\section{Impact of Spatial Smoothness on Empirical Generalization}
\label{sec:analysis}

While the previous section established HankelKANet's superior accuracy and lightweight nature, this section uncovers the underlying mechanisms driving this performance. Our analysis serves two primary objectives: First, to explain the root of this excellence by tracing it to physics-inspired inductive biases that inherently enforce spatial smoothness; Second, to demonstrate how this physics-guided regularization prevents overfitting to simulator artifacts, thereby enhancing cross-dataset generalization from simulated environments to real-world measurements.

\subsection{Hankel Operators as an Inductive Bias for Smoothness}

Generic activation functions (e.g., ReLU) and standard KAN basis functions (e.g., B-splines) lack physics-inspired decay constraints. Consequently, as demonstrated in the literature on the spectral bias of neural networks \cite{rahaman2019spectral}, these unconstrained models are highly susceptible to fitting high-frequency noise, making them prone to fabricating non-physical spatial oscillations in unobserved regions. In contrast, the gradient of the Hankel basis function is intrinsically guided by the asymptotic envelope of cylindrical wave equations, naturally mitigating such artifacts.

Consider a single-layer Hankel-KAN mapping acting on a local spatial displacement vector $\mathbf{r}$\footnote{In practical implementations, the input feature is the normalized radial distance $\widetilde{R} \propto \|\mathbf{r}\|_2$. We formulate the analysis using spatial displacement $\mathbf{r}$ to rigorously evaluate the 2D gradient $\nabla_{\mathbf{r}}$.}:
\begin{equation}
\varphi(\mathbf{r}) = w_0 \phi_0(\kappa \|\mathbf{r}\|_2) + w_1 \phi_1(\kappa \|\mathbf{r}\|_2),
\label{eq:single_layer_hankel}
\end{equation}
where $\phi_\nu(z) = \log(1 + |H_\nu^{(1)}(z)|)$ ($\nu \in \{0, 1\}$), $w_0, w_1 \in \mathbb{R}$ are trainable modal weights, and $\kappa \in \mathbb{R}^+$ is a learnable wavenumber.

\begin{proposition}[Parameter-Conditional Spatial Gradient Bound]
\label{prop:hankel_lipschitz}
For any spatial position $\mathbf{r}$ in the propagation regime ($\kappa \|\mathbf{r}\|_2 \gg 1$), the spatial gradient magnitude of $\varphi(\mathbf{r})$ is conditionally bounded by the modal parameters:
\begin{equation}
\|\nabla_{\mathbf{r}} \varphi(\mathbf{r})\|_2 \le \mathcal{C} \cdot \frac{\kappa^{1/2}}{\|\mathbf{r}\|_2^{1/2}} \big( |w_0| + |w_1| \big),
\label{eq:asymptotic_bound}
\end{equation}
where $\mathcal{C} = \sqrt{2/\pi}$ is a constant governed by the 2D cylindrical wave expansion.
\end{proposition}

\begin{proof}
Let $z = \kappa \|\mathbf{r}\|_2$. By the chain rule, $\nabla_{\mathbf{r}} \varphi(\mathbf{r}) = (d\varphi/dz) (\kappa\mathbf{r}^\top / \|\mathbf{r}\|_2)$, yielding $\|\nabla_{\mathbf{r}} \varphi(\mathbf{r})\|_2 = \kappa |d\varphi/dz|$. For each basis $\phi_\nu(z) = \log(1 + |H_\nu^{(1)}(z)|)$, using the contraction inequality and dynamic compression $(1+|H_\nu^{(1)}(z)|)^{-1} \le 1$, we obtain: $|d\phi_\nu/dz| \le |d H_\nu^{(1)}(z)/dz|$.
Using the Bessel recurrence relations $d H_0^{(1)}(z)/dz = -H_1^{(1)}(z)$ and $d H_1^{(1)}(z)/dz = \frac{1}{2} [ H_0^{(1)}(z) - H_2^{(1)}(z) ]$, applying the triangle inequality yields:
\begin{equation}
\|\nabla_{\mathbf{r}} \varphi(\mathbf{r})\|_2 \le \kappa \Big( |w_0| \big|H_1^{(1)}\big| + \frac{1}{2} |w_1| \big( \big|H_0^{(1)}\big| + \big|H_2^{(1)}\big| \big) \Big).
\label{eq:exact_grad_bound}
\end{equation}
Substituting the leading-order asymptotic decay $|H_\nu^{(1)}(z)| \sim \sqrt{2/(\pi z)}$ for $\nu \in \{0, 1, 2\}$ into Eq.~\eqref{eq:exact_grad_bound} directly yields Eq.~\eqref{eq:asymptotic_bound}.
\end{proof}

\textbf{Remark (Physics-Inspired Structural Regularization):} Eq.~\eqref{eq:asymptotic_bound} reveals a fundamental property: the spatial gradient naturally decays at a rate of $\mathcal{O}(\|\mathbf{r}\|_2^{-1/2})$. Functioning as a physics-inspired distance encoding, this inherent asymptotic envelope acts as a structural inductive bias, preventing the network from fabricating high-frequency spatial noise (spurious gradients) in the far field.

\subsection{Empirical Verification of Spatial Smoothness}
\label{subsec:empirical_verification}

To validate whether the theoretical gradient bound (Proposition~\ref{prop:hankel_lipschitz}) effectively regularizes the learned neural mappings, we first track the global geometric complexity over the training epochs. For a fair comparison across architectures, we evaluate the dimension-normalized DDE as:
\begin{equation}
GC_{\mathrm{avg}} = \frac{GC_{\mathrm{sum}}}{C \times H \times W},
\end{equation}
where $GC_{\mathrm{sum}}$ is the total DDE across the spatial domain, $C$ is the number of input channels, and $H=W=256$ are the spatial dimensions. 

\begin{figure*}[htbp]
\centering
\includegraphics[width=\linewidth]{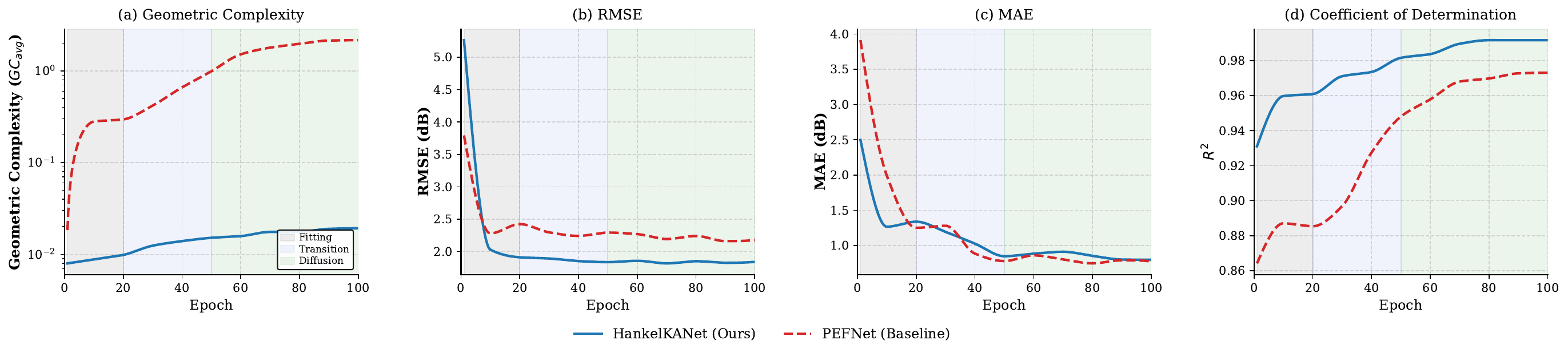}
\caption{Evolution of the Dimension-Normalized Geometric Complexity ($GC_{\mathrm{avg}}$) and prediction metrics over training epochs. Following [16], the shaded regions denote three optimization stages: \textit{fitting} (capturing dominant large-scale structures), \textit{transition} (refining finer details), and \textit{diffusion} (stabilizing the learned mapping). Unlike PEFNet, which experiences a rapid initial surge in complexity before stabilizing at a highly elevated level, HankelKANet maintains a consistently low $GC_{\mathrm{avg}}$ throughout training. Ultimately, HankelKANet settles at a value several orders of magnitude lower than PEFNet. This massive gap indicates that HankelKANet converges to an intrinsically smoother mapping that naturally aligns with large-scale propagation physics.}
\label{fig:gc_metrics_singlecol}
\end{figure*}

As shown in Fig.~\ref{fig:gc_metrics_singlecol}, training evolves through three distinct phases: fitting, transition, and diffusion \cite{toscano2025kkans}. Strikingly, while PEFNet's global geometric complexity surges dramatically during the early stages before plateauing at a high level, HankelKANet maintains a consistently flat and stable $GC_{\mathrm{avg}}$ profile throughout. By the end of training, HankelKANet's complexity remains orders of magnitude lower, indicating that it does not achieve high accuracy by memorizing irregular functional landscapes, but rather converges to an intrinsically smoother mapping.

\begin{figure}[htbp]
\centering
% ===== 图 (a) 2D 热力图 =====
\includegraphics[width=\linewidth]{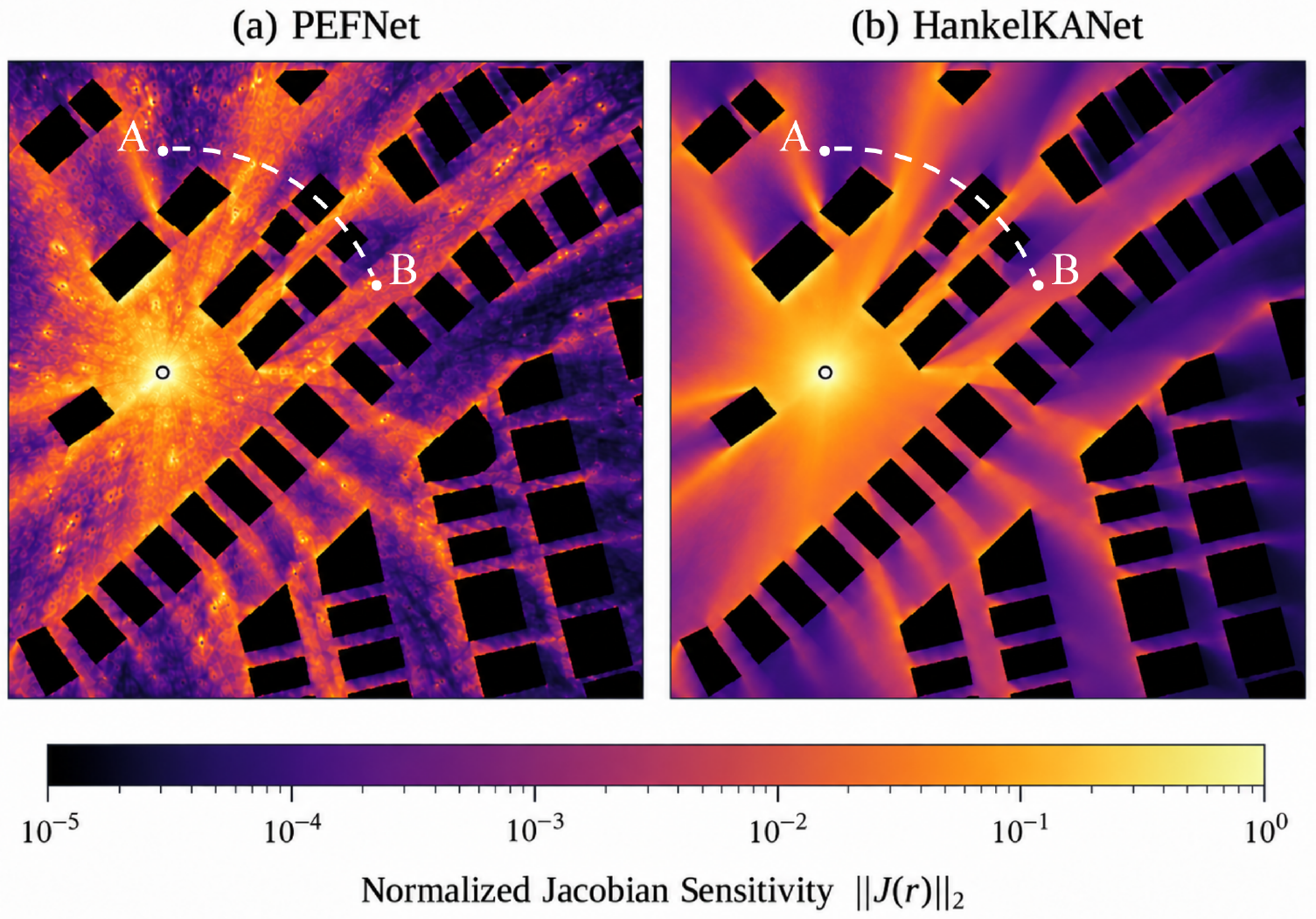}
\centerline{\small (a) 2D spatial heatmap of the Jacobian gradient norm}
\vspace{0.3cm} % 调整两张图之间的垂直间距

% ===== 图 (b) 1D 灵敏度曲线 =====
\includegraphics[width=\linewidth]{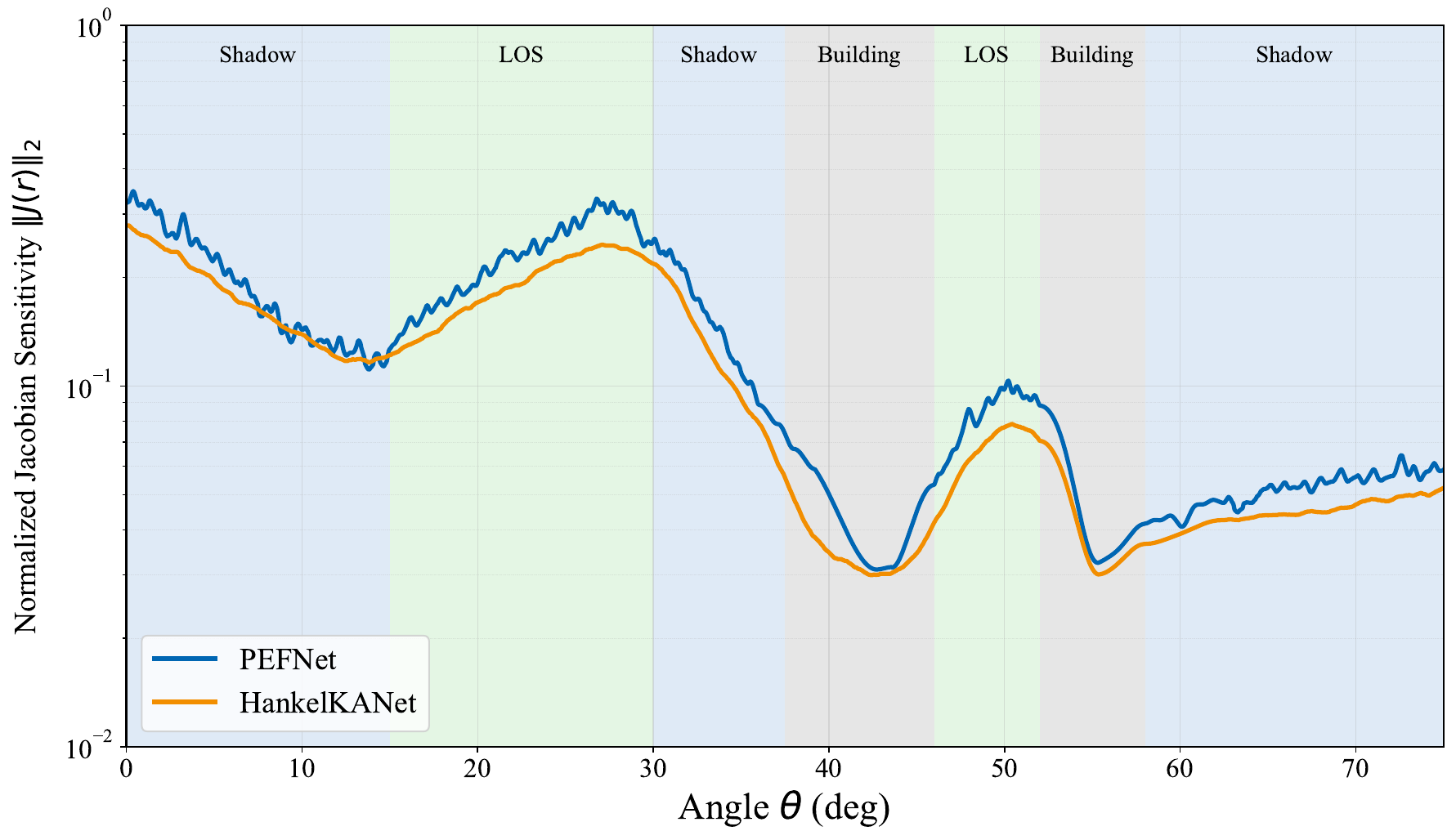}
\centerline{\small (b) 1D sensitivity profile along iso-distance arc $A \to B$}

\caption{Microscopic Jacobian Sensitivity visualization. (a) 2D spatial heatmap of the Jacobian gradient norm. (b) 1D sensitivity profile extracted along an iso-distance arc traversing both Line-of-Sight (LOS) and Non-Line-of-Sight (NLOS) regions. While both models react to diffraction boundaries, PEFNet suffers from pervasive high-frequency fluctuations across all regions. In contrast, HankelKANet imposes physics-inspired smoothness, yielding a stable sensitivity curve that aligns with Proposition 1.}
\label{fig:thermodynamic}
\end{figure}

This macroscopic stability is directly mirrored at the microscopic scene level. Fig.~\ref{fig:thermodynamic}(a) visualizes the spatial Jacobian sensitivity norm, $\|\nabla_{\mathbf{r}} \widehat{\mathrm{PL}}(\mathbf{r})\|_2$, for a representative scene. While PEFNet inherently produces erratic gradient flares and local hotspots, HankelKANet exhibits smooth, radially decaying contours originating from the transmitter. To rigorously isolate these behaviors from the dominant radial path loss decay, we extract a 1D sensitivity profile along an iso-distance arc $A \to B$ ($\|\mathbf{r}\|_2 = \text{const.}$) traversing both LOS and shadowed NLOS regions (Fig.~\ref{fig:thermodynamic}(b)). 

Given the constant radial distance along this arc, the theoretical gradient in unobstructed free space should remain continuous and stable. While both models correctly capture the macroscopic step-changes at diffraction boundaries (buildings), their continuous-space behaviors diverge fundamentally. PEFNet exhibits pervasive jagged fluctuations across all regions, corroborating the macroscopic complexity surge seen in Fig.~\ref{fig:gc_metrics_singlecol}. Without intrinsic physical constraints, standard networks relying on piecewise activations (e.g., ReLU) lack explicit Lipschitz-like bounds and tend to overfit discrete training points, inevitably generating non-physical, Runge-like spatial ripples \cite{gouk2021regularisation}. Conversely, HankelKANet maintains a remarkably smooth curve without spurious oscillations. This confirms that the embedded asymptotic envelope, functioning as a physics-inspired distance encoding, provides strong structural regularization. It inherently bounds the spatial gradient and naturally prevents the network from fabricating high-frequency spatial noise.

\subsection{From Spatial Smoothness to Real-World Generalization}
\label{subsec:real_world_transfer}

The structural spatial smoothness discussed in the previous subsection naturally extends to superior cross-scenario and cross-dataset generalizability. Standard unconstrained networks tend to overfit the specific artifacts of the simulated source domain (as evidenced by their non-physical spatial ripples), which degrades their performance on distinct target domains. By leveraging the physics-inspired inductive bias (Proposition 1) to maintain bounded spatial gradients, HankelKANet is less prone to such overfitting, allowing it to extract domain-invariant wave propagation characteristics.

\begin{table}[htbp]
\centering
\caption{Cross-Dataset Generalization Validation: Transfer Performance (RMSE in dB) on the Real-World Huawei RSRPSet}
\label{tab:response_transfer_validation}
\resizebox{\linewidth}{!}{
\begin{tabular}{lcccccc}
\toprule
\multirow{2}{*}{\textbf{Model}} & \multirow{2}{*}{\textbf{Pre-trained}} & \multicolumn{5}{c}{\textbf{Target Domain Data Budget ($\eta$)}} \\
\cmidrule(lr){3-7}
& & \textbf{0\% (Zero-shot)} & \textbf{1\%} & \textbf{5\%} & \textbf{10\%} & \textbf{20\%} \\
\midrule
Scratch Baseline & No  & - & 7.42 & 5.91 & 4.94 & 4.31 \\
\midrule
PEFNet & Yes & 6.43 & 5.73 & 4.92 & 4.48 & 4.12 \\
HankelKANet & Yes & \textbf{5.82} & \textbf{5.14} & \textbf{4.57} & \textbf{4.23} & \textbf{3.98} \\
\bottomrule
\end{tabular}
}
\end{table}

To empirically evaluate this cross-dataset adaptability, we transfer the models pre-trained on our simulated dataset to the real-world Huawei RSRPSet \cite{rsrpset16cpgmcm}. As summarized in Table~\ref{tab:response_transfer_validation}, HankelKANet consistently outperforms PEFNet across all target data budgets ($\eta$). The most notable indicator is the zero-shot scenario ($\eta = 0\%$): without seeing any field data, HankelKANet achieves an RMSE of $5.82$~dB, surpassing PEFNet ($6.43$~dB) and even outperforming the scratch baseline trained with $5\%$ real data ($5.91$~dB). This suggests that HankelKANet successfully transfers underlying physical decay characteristics rather than merely memorizing source-domain simulation features. Furthermore, during sparse adaptation, HankelKANet demonstrates superior data efficiency. Fine-tuning HankelKANet with only $5\%$ of the target data ($4.57$~dB) yields performance comparable to PEFNet using nearly double the budget ($10\%$ data, $4.48$~dB). 

Ultimately, these empirical results confirm that the structural spatial smoothness is not merely an aesthetic advantage, but a fundamental driver for robust real-world generalization and efficient target-domain adaptation.

\section{Conclusion}
This paper presented HankelKANet, a lightweight and physics-inspired neural network for path loss estimation. By integrating asymptotic electromagnetic decay characteristics to function as a physics-inspired distance encoding, the proposed architecture significantly streamlines the required input representations and the overall learning process. Experimental evaluations on both the simulated RadioMapSeer dataset and the real-world Huawei RSRPSet demonstrate that HankelKANet achieves state-of-the-art accuracy with minimal computational overhead. Beyond baseline performance, our analysis confirms that the embedded Hankel prior provides an effective structural regularization. This intrinsic spatial smoothness naturally prevents the network from overfitting to simulator artifacts, which serves as the fundamental driver for its robust cross-dataset generalization and high data efficiency during target-domain adaptation. Ultimately, HankelKANet offers an efficient, physics-guided, and highly practical solution for coverage planning and interference management in emerging 5G/6G networks.

\bibliographystyle{IEEEtran}
\bibliography{IEEEabrv,reference}

\begin{thebibliography}{10}
\providecommand{\url}[1]{#1}
\csname url@rmstyle\endcsname
\providecommand{\newblock}{\relax}
\providecommand{\bibinfo}[2]{#2}
\providecommand\BIBentrySTDinterwordspacing{\spaceskip=0pt\relax}
\providecommand\BIBentryALTinterwordstretchfactor{4}
\providecommand\BIBentryALTinterwordspacing{\spaceskip=\fontdimen2\font plus
\BIBentryALTinterwordstretchfactor\fontdimen3\font minus
  \fontdimen4\font\relax}
\providecommand\BIBforeignlanguage[2]{{%
\expandafter\ifx\csname l@#1\endcsname\relax
\typeout{** WARNING: IEEEtran.bst: No hyphenation pattern has been}%
\typeout{** loaded for the language `#1'. Using the pattern for}%
\typeout{** the default language instead.}%
\else
\language=\csname l@#1\endcsname
\fi
#2}}

\bibitem{phillips2012survey}
C.~Phillips, D.~Sicker, and D.~Grunwald, ``A survey of wireless path loss
  prediction and coverage mapping methods,'' \emph{IEEE Communications Surveys
  \& Tutorials}, vol.~15, no.~1, pp. 255--270, 2012.

\bibitem{joseph2023path}
I.~Joseph, E.~Oghu, and O.~O. Roberts, ``Path loss and models: A survey and
  future perspective for wireless communication networks,'' \emph{International
  Journal of Advanced Networking and Applications}, vol.~15, no.~02, pp.
  5892--5907, 2023.

\bibitem{ruan2025divided}
Y.~Ruan, P.~Chen, Z.~Yang, W.~Zhao, F.~Tang, Z.~Liu, and N.~Kato, ``The
  divided-qubo-based quantum algorithm for optimal wireless link scheduling,''
  \emph{IEEE Transactions on Emerging Topics in Computing}, 2025.

\bibitem{3gpp36873}
\BIBentryALTinterwordspacing
3GPP, ``Study on 3d channel model for lte,'' 3rd Generation Partnership Project
  (3GPP), Technical Report TR 36.873, 2015. [Online]. Available:
  \url{https://portal.3gpp.org/desktopmodules/Specifications/SpecificationDetails.aspx?specificationId=2574}
\BIBentrySTDinterwordspacing

\bibitem{3gpp38901}
\BIBentryALTinterwordspacing
------, ``Study on channel model for 5g new radio (nr),'' 3rd Generation
  Partnership Project (3GPP), Technical Report TR 38.901, 2018. [Online].
  Available:
  \url{https://portal.3gpp.org/desktopmodules/Specifications/SpecificationDetails.aspx?specificationId=3173}
\BIBentrySTDinterwordspacing

\bibitem{yun2015ray}
Z.~Yun and M.~F. Iskander, ``Ray tracing for radio propagation modeling:
  Principles and applications,'' \emph{IEEE access}, vol.~3, pp. 1089--1100,
  2015.

\bibitem{huang2018big}
J.~Huang, C.-X. Wang, L.~Bai, J.~Sun, Y.~Yang, J.~Li, O.~Tirkkonen, and M.-T.
  Zhou, ``A big data enabled channel model for 5g wireless communication
  systems,'' \emph{IEEE Transactions on Big Data}, vol.~6, no.~2, pp. 211--222,
  2018.

\bibitem{levie2021radiounet}
R.~Levie, {\c{C}}.~Yapar, G.~Kutyniok, and G.~Caire, ``Radiounet: Fast radio
  map estimation with convolutional neural networks,'' \emph{IEEE Transactions
  on Wireless Communications}, vol.~20, no.~6, pp. 4001--4015, 2021.

\bibitem{karniadakis2021physics}
G.~E. Karniadakis, I.~G. Kevrekidis, L.~Lu, P.~Perdikaris, S.~Wang, and
  L.~Yang, ``Physics-informed machine learning,'' \emph{Nature Reviews
  Physics}, vol.~3, no.~6, pp. 422--440, 2021.

\bibitem{meng2025physics}
C.~Meng, S.~Griesemer, D.~Cao, S.~Seo, and Y.~Liu, ``When physics meets machine
  learning: A survey of physics-informed machine learning,'' \emph{Machine
  Learning for Computational Science and Engineering}, vol.~1, no.~1, p.~20,
  2025.

\bibitem{raissi2019physics}
M.~Raissi, P.~Perdikaris, and G.~E. Karniadakis, ``Physics-informed neural
  networks: A deep learning framework for solving forward and inverse problems
  involving nonlinear partial differential equations,'' \emph{Journal of
  Computational physics}, vol. 378, pp. 686--707, 2019.

\bibitem{luo2025physics}
K.~Luo, J.~Zhao, Y.~Wang, J.~Li, J.~Wen, J.~Liang, H.~Soekmadji, and S.~Liao,
  ``Physics-informed neural networks for pde problems: A comprehensive
  review,'' \emph{Artificial Intelligence Review}, vol.~58, no.~10, p. 323,
  2025.

\bibitem{jiang2024physics}
F.~Jiang, T.~Li, X.~Lv, H.~Rui, and D.~Jin, ``Physics-informed neural networks
  for path loss estimation by solving electromagnetic integral equations,''
  \emph{IEEE Transactions on Wireless Communications}, 2024.

\bibitem{liu2024kan}
Z.~Liu, Y.~Wang, S.~Vaidya, F.~Ruehle, J.~Halverson, M.~Solja{\v{c}}i{\'c},
  T.~Y. Hou, and M.~Tegmark, ``Kan: Kolmogorov-arnold networks,'' \emph{arXiv
  preprint arXiv:2404.19756}, 2024.

\bibitem{toscano2025pinns}
J.~D. Toscano, V.~Oommen, A.~J. Varghese, Z.~Zou, N.~Ahmadi~Daryakenari, C.~Wu,
  and G.~E. Karniadakis, ``From pinns to pikans: Recent advances in
  physics-informed machine learning,'' \emph{Machine Learning for Computational
  Science and Engineering}, vol.~1, no.~1, pp. 1--43, 2025.

\bibitem{toscano2025kkans}
J.~D. Toscano, L.-L. Wang, and G.~E. Karniadakis, ``Kkans:
  Kurkova-kolmogorov-arnold networks and their learning dynamics,''
  \emph{Neural Networks}, vol. 191, p. 107831, 2025.

\bibitem{chew1999waves}
W.~C. Chew, \emph{Waves and fields in inhomogenous media}.\hskip 1em plus 0.5em
  minus 0.4em\relax John Wiley \& Sons, 1999.

\bibitem{wang2021understanding}
S.~Wang, Y.~Teng, and P.~Perdikaris, ``Understanding and mitigating gradient
  flow pathologies in physics-informed neural networks,'' \emph{SIAM Journal on
  Scientific Computing}, vol.~43, no.~5, pp. A3055--A3081, 2021.

\bibitem{rohrhofer2023data}
F.~M. Rohrhofer, S.~Posch, C.~G{\"o}{\ss}nitzer, and B.~C. Geiger, ``Data vs.
  physics: The apparent pareto front of physics-informed neural networks,''
  \emph{IEEE Access}, vol.~11, pp. 86\,252--86\,261, 2023.

\bibitem{wang2022and}
S.~Wang, X.~Yu, and P.~Perdikaris, ``When and why pinns fail to train: A neural
  tangent kernel perspective,'' \emph{Journal of Computational Physics}, vol.
  449, p. 110768, 2022.

\bibitem{bischof2025multi}
R.~Bischof and M.~A. Kraus, ``Multi-objective loss balancing for
  physics-informed deep learning,'' \emph{Computer Methods in Applied Mechanics
  and Engineering}, vol. 439, p. 117914, 2025.

\bibitem{zhang2002fast}
D.~Zhang, X.-C. Yuan, N.~Ngo, and P.~Shum, ``Fast hankel transform and its
  application for studying the propagation of cylindrical electromagnetic
  fields,'' \emph{Optics Express}, vol.~10, no.~12, pp. 521--525, 2002.

\bibitem{gonzalez2021advanced}
G.~Gonzalez, \emph{Advanced electromagnetic wave propagation methods}.\hskip
  1em plus 0.5em minus 0.4em\relax CRC Press, 2021.

\bibitem{kolmogorov1957representations}
A.~N. Kolmogorov, ``On the representations of continuous functions of many
  variables by superposition of continuous functions of one variable and
  addition,'' in \emph{Dokl. Akad. Nauk USSR}, vol. 114, 1957, pp. 953--956.

\bibitem{ta2024bsrbf}
H.-T. Ta, ``Bsrbf-kan: A combination of b-splines and radial basis functions in
  kolmogorov-arnold networks,'' in \emph{International Symposium on Information
  and Communication Technology}.\hskip 1em plus 0.5em minus 0.4em\relax
  Springer, 2024, pp. 3--15.

\bibitem{xu2025enhancing}
J.~Xu, Z.~Chen, J.~Li, S.~Yang, W.~Wang, X.~Hu, and E.~Ngai, ``Enhancing graph
  collaborative filtering with fourierkan feature transformation,'' in
  \emph{Proceedings of the 34th ACM International Conference on Information and
  Knowledge Management}, 2025, pp. 5376--5380.

\bibitem{seydi2024unveiling}
S.~T. Seydi, Z.~Bozorgasl, and H.~Chen, ``Unveiling the power of wavelets: A
  wavelet-based kolmogorov-arnold network for hyperspectral image
  classification,'' \emph{arXiv preprint arXiv:2406.07869}, 2024.

\bibitem{chao2026rbf}
Z.~Chao, X.~Liu, Z.~Wu, and X.~Li, ``Rbf-kan: Radial basis
  function-kolmogorov-arnold networks,'' \emph{IEEE Internet of Things
  Journal}, 2026.

\bibitem{seydi2024exploring}
S.~T. Seydi, ``Exploring the potential of polynomial basis functions in
  kolmogorov-arnold networks: A comparative study of different groups of
  polynomials,'' \emph{arXiv preprint arXiv:2406.02583}, 2024.

\bibitem{bodner2024convolutional}
A.~D. Bodner, A.~S. Tepsich, J.~N. Spolski, and S.~Pourteau, ``Convolutional
  kolmogorov-arnold networks,'' \emph{arXiv preprint arXiv:2406.13155}, 2024.

\bibitem{dherin2022neural}
B.~Dherin, M.~Munn, M.~Rosca, and D.~Barrett, ``Why neural networks find simple
  solutions: The many regularizers of geometric complexity,'' \emph{Advances in
  Neural Information Processing Systems}, vol.~35, pp. 2333--2349, 2022.

\bibitem{mitrea2020geometric}
D.~Mitrea, I.~Mitrea, and M.~Mitrea, ``Geometric harmonic analysis,''
  \emph{Book Manuscript}, 2020.

\bibitem{yueh2023theoretical}
M.-H. Yueh, ``Theoretical foundation of the stretch energy minimization for
  area-preserving simplicial mappings,'' \emph{SIAM Journal on Imaging
  Sciences}, vol.~16, no.~3, pp. 1142--1176, 2023.

\bibitem{levie2020radiomapseer}
``Radiomapseer dataset,'' \emph{https://radiomapseer.github.io/}, 2020.

\bibitem{rsrpset16cpgmcm}
\BIBentryALTinterwordspacing
Y.~Zheng, ``Rsrpset: the dataset of the 16th cpgmcm,'' 5, collected by Huawei
  for the 16th China Postgraduate Mathematical Contest in Modeling (CPGMCM),
  measured urban LTE RSRP path loss dataset. [Online]. Available:
  \url{https://ieee-dataport.org/documents/rsrpset-dataset-16th-cpgmcm}
\BIBentrySTDinterwordspacing

\bibitem{zhang2019path}
Y.~Zhang, J.~Wen, G.~Yang, Z.~He, and J.~Wang, ``Path loss prediction based on
  machine learning: Principle, method, and data expansion,'' \emph{Applied
  Sciences}, vol.~9, no.~9, p. 1908, 2019.

\bibitem{gomes2018empirical}
I.~R. Gomes, C.~R. Gomes, H.~S. Gomes, and G.~P. d.~S. Cavalcante, ``Empirical
  radio propagation model for dtv applied to non-homogeneous paths and
  different climates using machine learning techniques,'' \emph{PloS one},
  vol.~13, no.~3, p. e0194511, 2018.

\bibitem{rahaman2019spectral}
N.~Rahaman, A.~Baratin, D.~Arpit, F.~Draxler, M.~Lin, F.~Hamprecht, Y.~Bengio,
  and A.~Courville, ``On the spectral bias of neural networks,'' in
  \emph{International conference on machine learning}.\hskip 1em plus 0.5em
  minus 0.4em\relax PMLR, 2019, pp. 5301--5310.

\bibitem{gouk2021regularisation}
H.~Gouk, E.~Frank, B.~Pfahringer, and M.~J. Cree, ``Regularisation of neural
  networks by enforcing lipschitz continuity,'' \emph{Machine Learning}, vol.
  110, no.~2, pp. 393--416, 2021.

\end{thebibliography}
\vspace{-25pt}

\vfill

\end{document}